%% file: paper.tex
\documentclass[nohyperref]{article}

\usepackage{microtype}
\usepackage{graphicx}
\usepackage{subcaption}
\usepackage{wrapfig}
\usepackage{booktabs} % for professional tables

\usepackage{hyperref}

\usepackage[accepted]{icml2022}

\usepackage{amsmath}
\usepackage{amssymb}
\usepackage{mathtools}
\usepackage{amsthm}

\usepackage[capitalize,noabbrev]{cleveref}

\theoremstyle{plain}
\newtheorem{theorem}{Theorem}[section]
\newtheorem{proposition}[theorem]{Proposition}
\newtheorem{lemma}[theorem]{Lemma}
\newtheorem{corollary}[theorem]{Corollary}
\theoremstyle{definition}
\newtheorem{definition}[theorem]{Definition}
\newtheorem{assumption}[theorem]{Assumption}
\theoremstyle{remark}
\newtheorem{remark}[theorem]{Remark}

\usepackage[textsize=tiny]{todonotes}
\newcommand{\name}{PAIR}

\newcommand{\hide}[1]{}
\newcommand{\yf}[1]{#1}
\newcommand{\yw}[1]{\textcolor{red}{[Yi: #1]}}

\icmltitlerunning{Phasic Self-Imitative Reduction for Sparse-Reward Goal-Conditioned Reinforcement Learning}

\begin{document}

\twocolumn[
\icmltitle{Phasic Self-Imitative Reduction for \\
           Sparse-Reward Goal-Conditioned Reinforcement Learning}

% It is OKAY to include author information, even for blind
% submissions: the style file will automatically remove it for you
% unless you've provided the [accepted] option to the icml2022
% package.

% List of affiliations: The first argument should be a (short)
% identifier you will use later to specify author affiliations
% Academic affiliations should list Department, University, City, Region, Country
% Industry affiliations should list Company, City, Region, Country

% You can specify symbols, otherwise they are numbered in order.
% Ideally, you should not use this facility. Affiliations will be numbered
% in order of appearance and this is the preferred way.
\icmlsetsymbol{equal}{*}

\begin{icmlauthorlist}
\icmlauthor{Yunfei Li}{equal,iiis}
\icmlauthor{Tian Gao}{equal,iiis}
\icmlauthor{Jiaqi Yang}{ucb}
\icmlauthor{Huazhe Xu}{stanford}
\icmlauthor{Yi Wu}{iiis,sqz}
% \icmlauthor{Firstname6 Lastname6}{sch,yyy,comp}
% \icmlauthor{Firstname7 Lastname7}{comp}
%\icmlauthor{}{sch}
% \icmlauthor{Firstname8 Lastname8}{sch}
% \icmlauthor{Firstname8 Lastname8}{yyy,comp}
%\icmlauthor{}{sch}
%\icmlauthor{}{sch}
\end{icmlauthorlist}

\icmlaffiliation{iiis}{Institute for Interdisciplinary Information Sciences, Tsinghua University, Beijing, China}
\icmlaffiliation{ucb}{Department of Electrical Engineering and Computer Sciences, University of California, Berkeley, CA, USA}
\icmlaffiliation{stanford}{Stanford University, CA, USA}
\icmlaffiliation{sqz}{Shanghai Qi Zhi Institute, Shanghai, China}
% \icmlaffiliation{sch}{School of ZZZ, Institute of WWW, Location, Country}

\icmlcorrespondingauthor{Yunfei Li}{liyf20@mails.tsinghua.edu.cn}
\icmlcorrespondingauthor{Yi Wu}{jxwuyi@gmail.com}

% You may provide any keywords that you
% find helpful for describing your paper; these are used to populate
% the "keywords" metadata in the PDF but will not be shown in the document
\icmlkeywords{Machine Learning, ICML}

\vskip 0.3in
]

% this must go after the closing bracket ] following \twocolumn[ ...

% This command actually creates the footnote in the first column
% listing the affiliations and the copyright notice.
% The command takes one argument, which is text to display at the start of the footnote.
% The \icmlEqualContribution command is standard text for equal contribution.
% Remove it (just {}) if you do not need this facility.

%\printAffiliationsAndNotice{}  % leave blank if no need to mention equal contribution
\printAffiliationsAndNotice{\icmlEqualContribution} % otherwise use the standard text.

\begin{abstract}
\input{00abstract}
\end{abstract}

\section{Introduction}
\input{10intro}

\section{Related Work}
\input{20related}

\section{Preliminary}
\input{30prelim}

\section{Method}
\input{40method}

\section{Theoretical Analysis}\label{sec:theory}
\input{99theory}

\section{Experiment}
\input{50expr}

\section{Conclusion}
\input{60conclusion}

\section*{Acknowledgements}
We thank Jingzhao Zhang for valuable discussion on phasic optimization. Yi Wu is supported by 2030 Innovation Megaprojects of China (Programme on New Generation Artificial Intelligence) Grant No. 2021AAA0150000. 
\bibliography{reference}
\bibliographystyle{icml2022}

%%%%%%%%%%%%%%%%%%%%%%%%%%%%%%%%%%%%%%%%%%%%%%%%%%%%%%%%%%%%%%%%%%%%%%%%%%%%%%%
%%%%%%%%%%%%%%%%%%%%%%%%%%%%%%%%%%%%%%%%%%%%%%%%%%%%%%%%%%%%%%%%%%%%%%%%%%%%%%%
% APPENDIX
%%%%%%%%%%%%%%%%%%%%%%%%%%%%%%%%%%%%%%%%%%%%%%%%%%%%%%%%%%%%%%%%%%%%%%%%%%%%%%%
%%%%%%%%%%%%%%%%%%%%%%%%%%%%%%%%%%%%%%%%%%%%%%%%%%%%%%%%%%%%%%%%%%%%%%%%%%%%%%%
\newpage
\appendix
\onecolumn
The project webpage is at \url{https://sites.google.com/view/pair-gcrl}.

\section{Missing Proofs in \cref{sec:theory}}

\input{90theoryproof}
\input{80expr_detail}
% You can have as much text here as you want. The main body must be at most $8$ pages long.
% For the final version, one more page can be added.
% If you want, you can use an appendix like this one, even using the one-column format.
%%%%%%%%%%%%%%%%%%%%%%%%%%%%%%%%%%%%%%%%%%%%%%%%%%%%%%%%%%%%%%%%%%%%%%%%%%%%%%%
%%%%%%%%%%%%%%%%%%%%%%%%%%%%%%%%%%%%%%%%%%%%%%%%%%%%%%%%%%%%%%%%%%%%%%%%%%%%%%%

\end{document}

%% file: 00abstract.tex
It has been a recent trend to leverage the power of supervised learning (SL) towards more effective reinforcement learning (RL) methods. %In this work, 
We propose a novel phasic approach by alternating online RL and offline SL for tackling sparse-reward goal-conditioned problems.  In the online phase, we perform RL training and collect rollout data while in the offline phase, we perform SL on those successful trajectories from the dataset. To further improve sample efficiency, we adopt additional techniques in the online phase including task reduction to generate more feasible trajectories and a value-
difference-based intrinsic reward to alleviate the sparse-reward issue. We call this overall algorithm, \emph{\underline{P}h\underline{A}sic self-\underline{I}mitative \underline{R}eduction} (\name). %{\name} is compatible with various online and offline RL methods and 
{\name} substantially outperforms both non-phasic RL and phasic SL baselines on sparse-reward goal-conditioned robotic control problems, including a challenging stacking task. {\name} is the first RL method that learns to stack \emph{6 cubes} with only \emph{0/1 success rewards} from scratch.

%Moreover, we also for the first time demonstrate that a simple {\name} implementation, i.e., using PPO for online RL and behavior cloning for offline SL, learns to successfully stack 6 blocks with only 0-1 sparse reward. 

%% file: 10intro.tex
Despite great advances achieved by deep reinforcement learning (RL) in a wide range of application domains such as playing games~\cite{mnih2015human,schrittwieser2020mastering}, controlling robots~\cite{lillicrap2016continuous,hwangbo2019learning,akkaya2019solving}, and solving scientific problems~\cite{jeon2020autonomous}, deep RL methods have been empirically shown to be brittle and extremely sensitive to hyper-parameter tuning~\cite{tucker2018mirage,Ilyas2020A,Engstrom2020Implementation,yu2021surprising,andrychowicz2021what}, which largely limits the practice use of deep RL in many real-world scenarios. On the contrary, supervised learning (SL) provides another learning paradigm by imitating given demonstrations, which is much simpler for tuning and typically results in a much more steady optimization process~\cite{lynch2020learning,ghosh2020learning}. 
Inspired by the success of training powerful fundamental models by SL~\cite{larochelle20language,dosovitskiy2020image,jumper2021highly}, it has also been a recent trend in RL to leverage the power of SL to develop more powerful and stable deep RL algorithms~\cite{levine21understanding}.

\begin{figure}[t]
    \centering
    \includegraphics[width=1.0\linewidth]{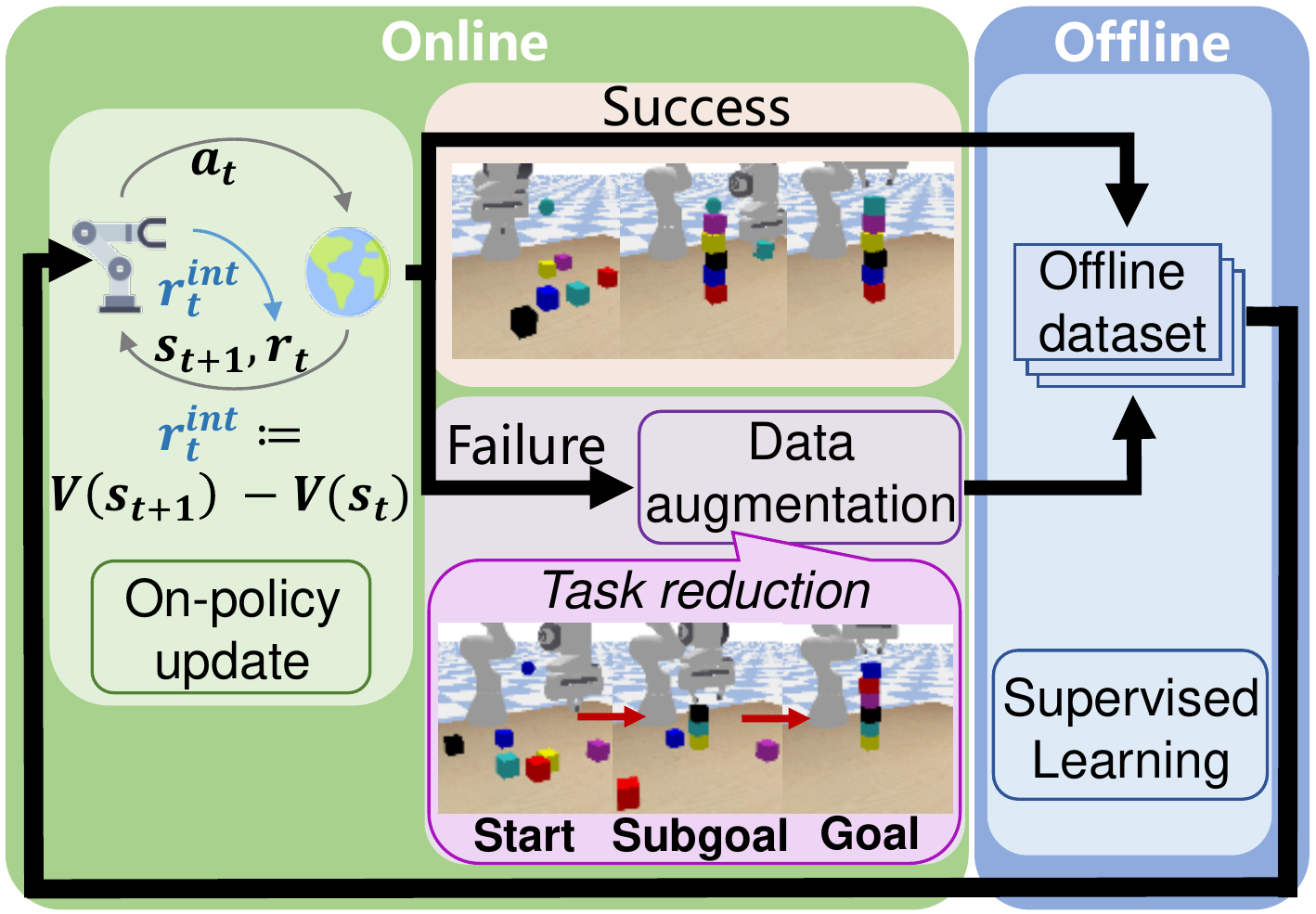}
    \caption{Overall workflow of {\name}. {\name} iteratively alternates between online RL and offline SL phases. During online phases, the agent is trained using both environment reward and a value-difference-based intrinsic reward. Meantime, it collects trajectories and augments them with task reduction into successful demonstrations for offline training. In the offline phases, the agent runs supervised learning to improve its policy using the generated dataset.}
    \vspace{-2mm}
    \label{fig:teaser}
\end{figure}
%Despite the great advances achieved by deep reinforcement learning, ranging from game playing, robot learning
%1. Despite the success of RL, RL has shown to be sensitive to hyper-param and subtle in practice. Supervised learning is much more robust and general. It has also been a recent trend to adopt SL techniques to tackle RL problems

One representative line of research that incorporates SL into RL is offline RL, which assumes that a large offline dataset of transition data is available and solely performs learning on the dataset without interacting with the environment~\cite{lange2012batch,wu2019behavior,levine2020offline,kumar20cql,fujimoto2021minimalist}.
%Behavior cloning (BC), i.e., directly performing SL over demonstration actions, is perhaps the simplest offline RL method. More recent advances have additionally incorporate modeling over accumulative returns\yw{DT; awac} or transition dynamics~\yw{muzero? model-based offline RL} for more powerful learning capabilities. \yw{probably don't need this sentence}
 %More recent methods additionally models accumulative returns~\yw{DT; AWAC;} or transition dynamics~\yw{muzero?}, which allows for the use of non-expert or even randomly collected dataset. 
However, both empirical~\cite{yu2021conservative} and theoretical~\cite{rashidinejad2021bridging} evidence suggests that the success of recent offline RL methods rely on the quality of the dataset. Accordingly, popular offline RL datasets are typically constructed by human experts, which may not be feasible for many real-world problems. 
Besides, offline RL suffers from a final performance gap compared with online RL algorithms.
%Moreover, offline RL works 

%most offline RL works do not consider fine-tuning after the offline training process and suffer from a final performance gap compared with online RL.

%2. A representative category of algorithms is offline RL methods, which follows the advances in NLP by directly learn on a fixed dataset without interacting with the environment. However, the dataset is typically from expert or carefully constructed (i.e., to cover the entire space, cite Jiantao Jiao's paper), which may not be feasible in many practical problems. Moreover, typical offline RL algorithms do not consider fine-tuning after the pretraing phase and typically suffer from a performance gap compared with online RL.

Another representative line of research is self-imitation learning (SIL)~\cite{oh2018self}, which combines SL and RL in a completely online fashion. SIL directly performs SL over selected self-generated rollout trajectories by treating the SL objective as an auxiliary loss and optimizing it jointly with the standard RL objective. 
SIL does not require a dataset in advance. However, optimizing a mixed objective consisting of both RL and SL makes the optimization process even more brittle and requires substantial efforts of parameter tuning to achieve best empirical performance. 
%works in a online fashion by simultaneously performing SL on selective past trajectories and treated the supervised objective as an auxiliary loss in addition to the standard RL objective. However, these methods inherit all the disadvantages of RL methods, i.e., sensitivity on hyper-param, and becomes even more brittle in practice due to the mixed learning objective. 

In this paper, we propose a simple phasic approach, \emph{\underline{P}h\underline{A}sic self-\underline{I}mitative \underline{R}eduction} (\name), to effectively balance both RL and SL training for goal-conditioned sparse-reward problems. The main principle of {\name} is to alternate between RL and SL: in the RL phase, we solely perform standard online RL training for optimization stability and collect rollout trajectories as a dataset for the offline SL phase; while in the SL phase, we pick out successful trajectories as SL signals and run imitation learning to improve policy.
To improve sample efficiency, {\name} also includes two additional techniques in the RL phase: 1) a value-difference-based intrinsic reward that alleviates the sparse-reward issue, and 2) \emph{task reduction}, a data augmentation technique that largely increases the total number of successful trajectories, especially for hard compositional tasks. 
Our theoretical analysis suggests that task reduction can converge \emph{exponentially} faster than the vanilla phasic approach that does not use task reduction. %version of our phasic appraoch with. 

We implement {\name} with Proximal Policy Optimization (PPO) for RL and behavior cloning for SL and we conduct experiments on a variety of goal-conditioned control problems, including relatively simple benchmarks such as pushing and ant-maze navigation, and a challenging sparse-reward cube-stacking task. {\name} substantially outperforms all the baselines including non-phasic online methods, which jointly perform self-imitation and RL training, and phasic SL methods, which only perform supervised learning on self-generated trajectories~\cite{ghosh2020learning}. 
We highlight that {\name} successfully learns to stack 6 cubes with 0/1 rewards from scratch. To our knowledge, {\name} is the \emph{first} deep RL method that could solve this challenging sparse reward task.

%% file: 20related.tex
% \yw{1. goal-conditioned RL}
\textbf{Goal-conditioned RL:}
We study goal-conditioned reinforcement learning~\cite{kaelbling1993learning} with sparse reward in this work. Goal-conditioned RL enables one agent to solve a variety of tasks by predicting actions given both observations and goals, and is studied in a number of works~\cite{schaul2015universal,nair2018visual,pong2018temporal,veeriah2018many,zhao2019maximum,eysenbach2020c}. Although some techniques like relabeling~\cite{andrychowicz2017hindsight} are proposed to address the sparse reward issue when learning goal-conditioned policies, there are still challenges in long-horizon problems~\cite{nasiriany2019planning}.

% \yw{2. offline RL}
\textbf{Offline reinforcement learning:}
Offline RL~\cite{lange2012batch,levine2020offline} is a popular line of research that incorporates SL into RL, which studies extracting a good policy from a fixed transition dataset. A large portion of offline RL methods focus on regularized dynamic programming (e.g., Q-learning)~\cite{kumar20cql,fujimoto2021minimalist}, with the constraint that the resulting policy does not deviate too much from the behavior policy in the dataset~\cite{wu2019behavior,peng2019advantage,kostrikov2021iql}. Some other works directly treat policy learning as \yf{a supervised learning} % a sequence generation 
problem, and learn the policy in a conditioned behavior cloning manner~\cite{chen2021decision,janner2021offline,furuta2021generalized,emmons2021rvs}, \yf{which can be considered as special cases of upside down RL and reward-conditioned policies~\cite{schmidhuber2019reinforcement,kumar2019reward}}. There are also methods that learn transition dynamics from offline data before 
extracting policies in a model-based way~\cite{matsushima2020deployment,kidambi2020morel}. Some works also consider a single online fine-tuning phase after offline learning~\cite{nair2020awac,lu2021awopt,mao2022moore,uchendu2022jump} while we repeatedly alternate between offline and online training.
% \yw{talk about more technical things! from bc to advanced methods modeling accumulative returns and transition dynamics (model-based offline-rl)}

% \yw{3. imitation learning in RL (including GCSL, PPG for phasic motivation)}
\textbf{Imitation learning in RL:} Imitation learning is a framework for learning policies from demonstrations, which has been shown to largely improve the sample complexity of RL methods~\cite{hester2018deep,rajeswaran18dexterous} and help overcome exploration issues~\cite{nair2018overcoming}. Self-imitation learning (SIL)~\cite{oh2018self}, which imitates good data rolled out by the RL agent itself, does not require any external dataset and  has been shown to help exploration in sparse reward tasks. Our method also performs imitation learning over self-generated data. %, and we also adopt it in our framework. Oftentimes, 
Self-imitation objective is optimized jointly with the RL objective, while we propose to perform SL and RL separately in two phases. The idea of substituting joint optimization with iterative training for minimal interference between different objectives is similar to phasic policy gradient~\cite{cobbe21ppg}. Goal-conditioned supervised learning (GCSL)~\cite{ghosh2020learning} is perhaps the most related work to ours. GCSL repeatedly performs imitation learning on self collected relabeled data without any RL updates while we leverage the power from both RL and SL and adopt further task reduction for enhanced data augmentation.

%, whichrepeatedly runs imitation learning on self-collected relabeled data without any RL updates, while we focus on combining RL and SL for good performances. % More related SIL, GCSL

% \yw{4. (optional) mention some sparse reward RL related thing. or compositional task. but this is pretty orthogonal}

% \textbf{Offline reinforcement learning:}
% Another popular paradigm that utilizes batch of transition data for reinforcement learning is offline reinforcement learning~\cite{levine2020offline}. Offline RL works typically focus on learning a good RL policy from a fixed dataset without consideration of further online interactions~\cite{kumar20cql,fujimoto2021minimalist,chen2021decision}, or amenable for online tuning only once~\cite{nair2020awac,lu2021awopt,kostrikov2021iql,lee2021offline}.
% \yf{Relation with our method: }

\textbf{Sparse reward problems:}
% \yf{methods to solve sparse reward problems, including HRL, option, exploration methods, credit assignment} 
There are orthogonal interests in solving long horizon sparse reward with hierarchical modeling~\cite{stolle2002learning,kulkarni2016hierarchical,bacon2017option,nachum2018data} or by encouraging exploration~\cite{singh2005intrinsically,burda2018exploration,Badia2020Never,ecoffet2021first}. 
% We are specifically interested in tackling sparse reward problems with compositional structure in this work. Hierarchical reinforcement learning (HRL) is a popular framework that embeds hierarchy into the policy architecture. 
% Another lines of work focus on enhancing exploration ability of the agent in sparse reward scenarios with intrinsic motivation~\cite{} or re-visiting states~\cite{}. 
% There are other works that create more dense reward signals with data relabeling~\cite{andrychowicz2017hindsight,li2020generalized}, synthetic credit assignment~\cite{raposo2021synthetic} or auxiliary training objectives~\cite{riedmiller2018learning}.
Our framework can be also viewed as an effective solution to tackle challenging sparse-reward compositional problems.
%Our framework is combined with self-imitative task reduction~\cite{li2020solving} to solve sparse reward compositional problems.

%% file: 30prelim.tex
%\yf{MDP notations, sparse reward task setting; self-imitation and RL objectives}
%\jiaqi{Is it about sparse reward, or it is about goal reaching?}

We consider the setting of goal-conditioned Markov decision process with 0/1 sparse rewards defined by $(\mathcal{S}, \mathcal{A}, P(s'|s,a), \mathcal{G}, r(s,a,g), \rho_0, \gamma)$. 
$\mathcal{S}$ is the state space, $\mathcal{A}$ is the action space, $\mathcal{G}$ is the goal space and $\gamma$ is the discounted factor. $P(s'|s,a)$ denotes transition probability from state $s$ to $s'$ after taking action $a$. The reward function $r(s,a,g)$ is 1 only if the goal $g$ is reached at the current state $s$ within some precision threshold and otherwise 0. %is defined as $r(s,a,g) = \mathbb{I}(||s - g|| < \delta)$, which is non-zero only if the state $s$ is close to the goal $g$ within some tolerance threshold $\delta$. 
In the beginning of each episode, the initial state $s_0$ is sampled from a distribution $\rho_0$ and a goal $g$ is sampled from the goal space. 
An episode terminates when the goal is achieved or it reaches a maximum number of steps. 

The agent is represented as a goal-conditioned stochastic policy $\pi_{\theta}(a|s, g)$ parametrized by $\theta$. The optimal policy $\pi_{\theta^{\star}}$ should maximize the objective $J(\theta)$ defined by the expected discounted cumulative reward over all the goals, i.e., 
\begin{align}\label{eq:rl}
J(\theta) = J(\pi_\theta) =\mathbb{E}_{g\in\mathcal{G},a_t\sim \pi_{\theta}}\left[\sum_t \gamma^t r(s_t, a_t, g)\right].
\end{align}

\textbf{Policy gradient } optimizes $J(\theta)$ via the gradient computation 
\begin{align*}
\nabla J(\theta)=\mathbb{E}_{g,a_t}\left[\sum_t (R_t-V_{\psi}(s_t,g))\nabla\log \pi(a_t|s_t,g)\right],
\end{align*}
where $V_{\psi}(s_t,g)$ is the goal-conditioned value function parameterized by $\psi$ and $R_t$ denotes the discounted return starting from time $t$ on the current trajectory.
Note that in our goal-conditioned setting with 0/1 rewards, $R_t$ will be either\footnote{More precisely, the return $R_t$ will be 0 or close to 1 due to the discount factor $\gamma$.} 0 or 1 and the value function $V_{\psi}(s_t,g)$ can be approximately interpreted as the discounted success rate from state $s_t$ towards goal $g$.

%For actor-critic algorithms, there is also a universal value function approximator $V_{\psi}(s, g)$ parametrized by $\psi$.

\textbf{Goal-conditioned imitation learning} optimizes a policy $\pi_\theta$ by running supervised learning over a given demonstration dataset $\mathcal{D}=\{\tau:(g;s_0,a_0,s_1,a_1,\ldots)\}$, where $\tau$ denotes a single trajectory. The supervised learning loss $L(\theta)$ is typically defined by
\begin{align}\label{eq:bc}
L(\theta)=-\mathbb{E}_{(g;s,a)\in \mathcal{D}} \left[w(s,a,g)\log \pi(a|s,g)\right],
\end{align}
where $w(s,a,g)$ is some sample weight. Behavior cloning (BC) simply sets $w(s,a,g)$ as $1$ while more advanced methods may have other choices. 
Note that self-imitation learning (SIL) is a paradigm that jointly optimizes the RL objective $J(\theta)$ and the SL objective $L(\theta)$ in an online fashion.

\iffalse
jointly optimizes an auxiliary loss $\mathcal{L}^{SL}$ with the original RL objective $\mathcal{L}^{RL}$, where $\mathcal{L}^{SL}$ is typically defined as
\begin{align*}
    \mathcal{L}^{SL} = -\mathbb{E}_{(s, a)\sim \mathcal{D}^{SL}} \left[w(s, a)\log \pi(a|s)\right].
\end{align*}
$w(s, a)$ denotes the weight for each state-action pair from the supervised dataset $\mathcal{D}^{SL}$. In this work, $\mathcal{D}^{SL}$ contains only successful trajectories and we use $w(s, a) = \exp(A(s, a))$. 
\yw{You don't need to talk about ``in this work'' in the preliminary section. Just talk about background knowledge.}

\textbf{Task reduction:} We adopt task reduction technique to collect more feasible data leveraging the compositionality of the problems. When the current policy fails to achieve a goal $g$ from an initial state $s_0$, we search for an intermediate state $s_{B}$ so that the two sub-tasks from $s_0$ to $s_B$ and from $s_B$ and $g$ are both easy to solve. The optimal intermediate state $s_{B}^{\star}$ satisfies
\begin{align*}
    s_B^{\star} = \arg\max_{s_B} V(s_0, s_B) \oplus V(s_B, g),
\end{align*}
where $\oplus$ is an operator that composes two values, and can be ``min'', ``average'' or ``product'' in practice.
The two trajectories are then concatenated and relabeled as successful demonstration of the original unsolved task. 

\yw{goal-conditioned MDP/RL; self-imitation/offline RL (difference is offline RL also consider accumulative return)? task reduction (describe the math in a few lines). }

\fi

%% file: 40method.tex
% \yf{Phasic RL and SL; Leveraging compositionality with task reduction; }

%In this section, we present our phasic solution that combines online reinforcement learning and offline supervised learning. 
Our phasic solution {\name} consists of 3 components, including the online RL phase (Sec.~\ref{sec:rl-phase}), which also collects rollout trajectories, task reduction (Sec.~\ref{sec:task-reduction}) as a mean for data augmentation, and the offline phase (Sec.~\ref{sec:offline}), which performs SL on self-generated demonstrations. We summarize the overall algorithm in Sec.~\ref{sec:pair}.

\subsection{RL Phase with Intrinsic Rewards}\label{sec:rl-phase}

The RL phase follows any standard online RL training. Specifically in our work, we adopt PPO~\cite{schulman17ppo} as our RL algorithm, which trains both policy $\pi_\theta$ and value function $V_\psi$ using rollout trajectories.
%, which trains a policy network $\pi_\theta$ and a value function $V_\phi(s)$ as critic. 
When a trajectory $\tau=(g;s_t,a_t)$ is successful, i.e., goal $g$ is reached, we keep this trajectory $\tau$ as a positive demonstration towards goal $g$ in the dataset $\mathcal{D}$ for the offline phase.

%When the current learning policy $\pi_\theta$ generates a trajectory $\tau$ where the given goal is successful reached, we keep this trajectory $\tau$ as a positive data for the offline phase.

%\yw{rewrite the following paragraph.}

% \begin{figure}
%     \centering
%     \begin{subfigure}{0.32\linewidth}
%         \includegraphics[width=\linewidth]{fig/stacking/catastrophic/img_start.png}
%         \caption{$V:0.86$}
%     \end{subfigure}
%     \begin{subfigure}{0.32\linewidth}
%         \includegraphics[width=\linewidth]{fig/stacking/catastrophic/img_bad.png}
%         \caption{$a:0.7, V:0.41$}
%     \end{subfigure}
%     \begin{subfigure}{0.32\linewidth}
%         \includegraphics[width=\linewidth]{fig/stacking/catastrophic/img_good.png}
%         \caption{$a:0.0, V:0.88$}
%     \end{subfigure}
%     \caption{Starting from the same state in (a), a tiny mistake in an action could lead to dramatically different outcomes: in (b) some cubes are knocked down thus the agent can rarely succeed within time budget, while in (c) the agent gets around the tower and is approaching the last cube to stack. The two actions only differ in one dimension corresponding to hand movement in $x$-axis. }
%     \label{fig:method:vd}
% \end{figure}
\begin{figure}
    \centering
    \includegraphics[width=\linewidth]{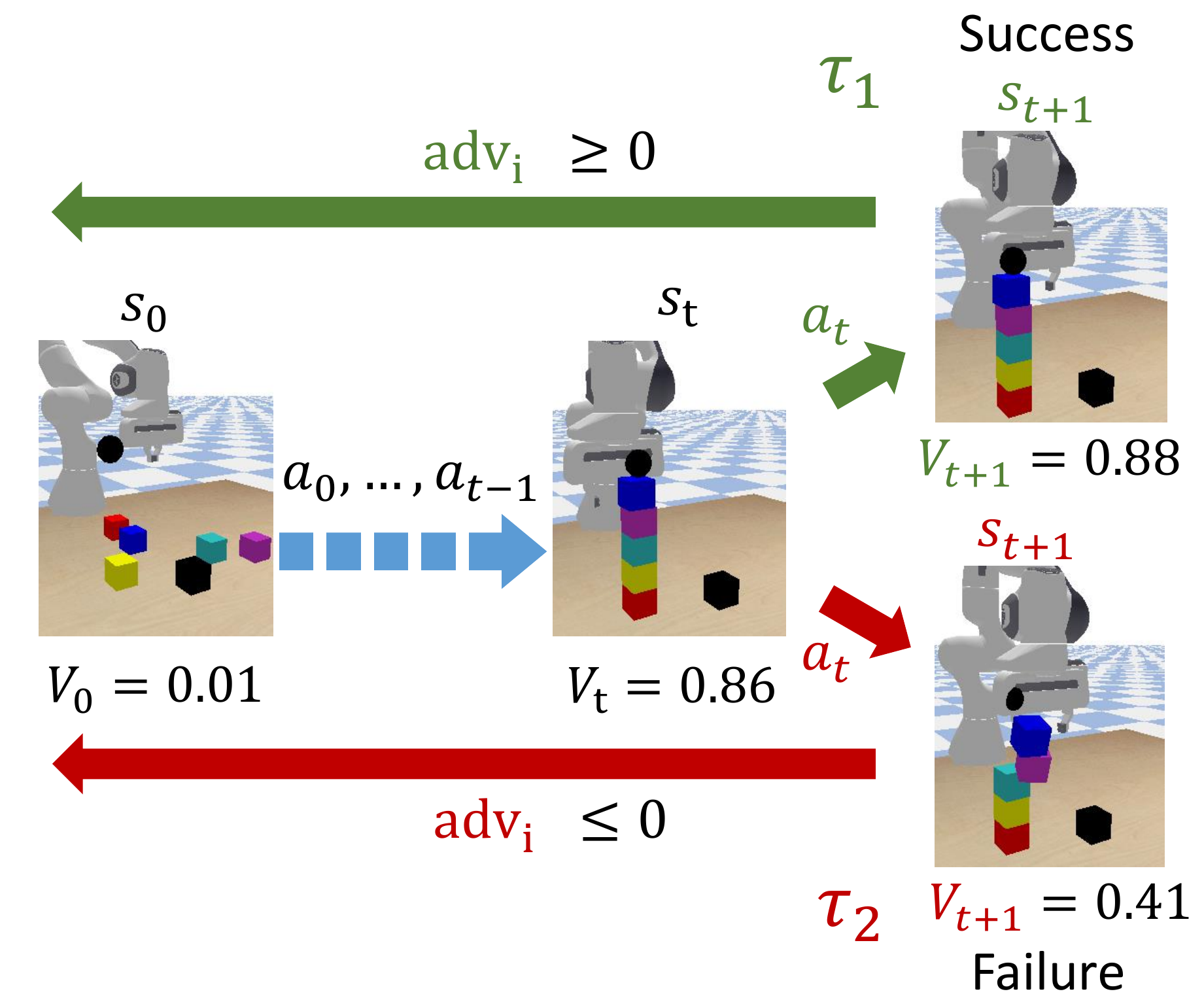}
    \caption{Motivation of value-difference intrinsic reward. Trajectories $\tau_1$ and $\tau_2$ only differ at a single action $a_t$. $\tau_1$ succeeds and will have positive advantage. $\tau_2$ fails since $a_t$ knocks down the cubes, then all the transitions in $\tau_2$ will have negative advantages due to 0/1 goal-conditioned reward, even if early actions are perfect.}
    \vspace{-2mm}
    \label{fig:method:r_int}
\end{figure}

\textbf{Value-difference-based intrinsic rewards:}

The sparse-reward issue is a significant challenge for online RL training, which can yield substantially high variance in gradients.
In particular, let's assume $\gamma=1$ for simplicity and consider a single trajectory $\tau=(g;s_t,a_t)$. In our goal-conditioned setting, the return $R_t$ on $\tau$ will be binary and the value function $V_{\psi}(s_t,g)$ will be approximately between 0 and 1. Hence, the advantage function for state-action pair $(s_t,a_t)\in \tau$, i.e., $R_t-V(s_t)$, will be either positive or negative for the entire trajectory $\tau$.
%Now let's assume $\gamma=1$ for simplicity \tian{move this assumption to the front of ``binary $R_t$’} and 
Imagine a concrete example as illustrated in Fig.~\ref{fig:method:r_int}, where we have two trajectories, a successful trajectory $\tau_1$ and a failed trajectory $\tau_2$. These two trajectories only differ at the final timestep $s_{t+1}$, which may be due to some random exploration noise at action $a_t$. However, only due to a single mistake, all the state-action pairs from $\tau_2$ will have negative advantages, even though most of the actions in $\tau_2$ are indeed approaching the desired target. Likewise, in a successful trajectory, it is also possible that some single action is poor but eventually the subsequent actions fix this early mistake.

%In this setting, the return $R(s)$ on states from a trajectory $\tau$ will be either 0 or 1 and therefore the value function would be ideally between 0 or 1. Note that the advantage function, i.e., $A(s,a)=R(s)-V(s)$, will be positive for any state-action pair over a successful trajectory and be negative for any failed trajectory. This isn't satisfying since any small error occurred during the rollout trajectory may lead to a drastic change over the sign of advantage function. Imagine in a $N$-cube stacking task where $N-1$ cubes have been successfully stacked already, unfortunately due to exploration noise, you accidentally break the partially accomplish tower during the placement of the final cube. In this example, even though all the actions for the previous $N-1$ cubes are perfect, due to the small mistake in at the end of the trajectory, all the actions over this trajectory will be penalized due to the sparse-reward setting. 

Besides the trajectory-based advantage computation, it will be beneficial to have some effective signal of whether a transition $(s_t,a_t,s_{t+1})$ is properly ``\emph{approaching}'' the desired goal $g$ or not.
Our suggestion is to use $V_{\psi}(s,g)$ as an empirical measure: if $a_t$ is a good action, it should lead to a higher state value, i.e., $V_{\psi}(s_{t+1},g)-V_{\psi}(s_t,g)>0$, which indicates that $a_t$ moves to a state with a higher success rate; similarly, a poor action will result in a value drop, i.e., $V_{\psi}(s_{t+1},g)-V_{\psi}(s_t,g)<0$.
Accordingly, we propose to adopt value difference as an intrinsic reward $r^{\textrm{int}}$ for stabilizing goal-conditioned RL training as follows 
\begin{align*} 
    %& r^{\textrm{d}}(s_t, a_t, g) = r(s_t, a_t, g) + \alpha r^{\textrm{int}}, \\
    & r^{\textrm{int}}(s_t,a_t,g) := V_{\psi}(s_{t+1}, g) - V_{\psi}(s_t, g).
\end{align*}

\iffalse
Is there any effective signal that can be used to empirical measure whether a transition $(s,a,s')$ is indeed approaching the given goal $g$? Our suggestion is to use the value function $V(s,g)$. Note that in our goal-conditioned setting, $V(s,g)$ can be interpreted a discounted success rate towards the goal $g$. Hence, if the action $a$ leads to a state $s'$ with a higher success rate $V(s',g)$ than the current one $V(s,g)$, the action $a$ should be very likely be promoted --- even if the entire trajectory fails. Moreover, if an action $a$ leads to a drastic success rate drop, the action $a$ should be additionally punished. Hence, we propose to use the value difference $V(s',g)-V(s,g)$ as an intrinsic reward as an indication of how effective an action $a$ is. More specifically,
\begin{align*}
    & r^{\textrm{d}}(s_t, a_t, g) = r(s_t, a_t, g) + \alpha r^{\textrm{int}}, \\
    & r^{\textrm{int}} := V(s_{t+1}, g) - V(s_t, g).
\end{align*}
% \yw{change $V(s,a,g)$ to $Q(s,a,g)$????}
\fi

% \yf{comment: our final implementation is actually a single-head version. Double head also works. Maybe remove this paragraph? }
We remark that $r^{\textrm{int}}$ relies on an accurately learned value function $V_{\psi}(s,g)$. Therefore, we suggest to only train $V_{\psi}(s,g)$ over the sparse goal-conditioned rewards while using another value head to fit the intrinsic rewards for critic learning. Similar techniques have been previously explored in \cite{burda2018exploration} as well.

%in the implementation that $V(s,g)$ is only trained on the extrinsic success rewards and we use an additional value head to fit the accumulative return from the intrinsic reward for GAE computation in PPO.

\subsection{Task Reduction as Data Augmentation}\label{sec:task-reduction}
\begin{figure}
    \centering
    \includegraphics[width=0.49\linewidth]{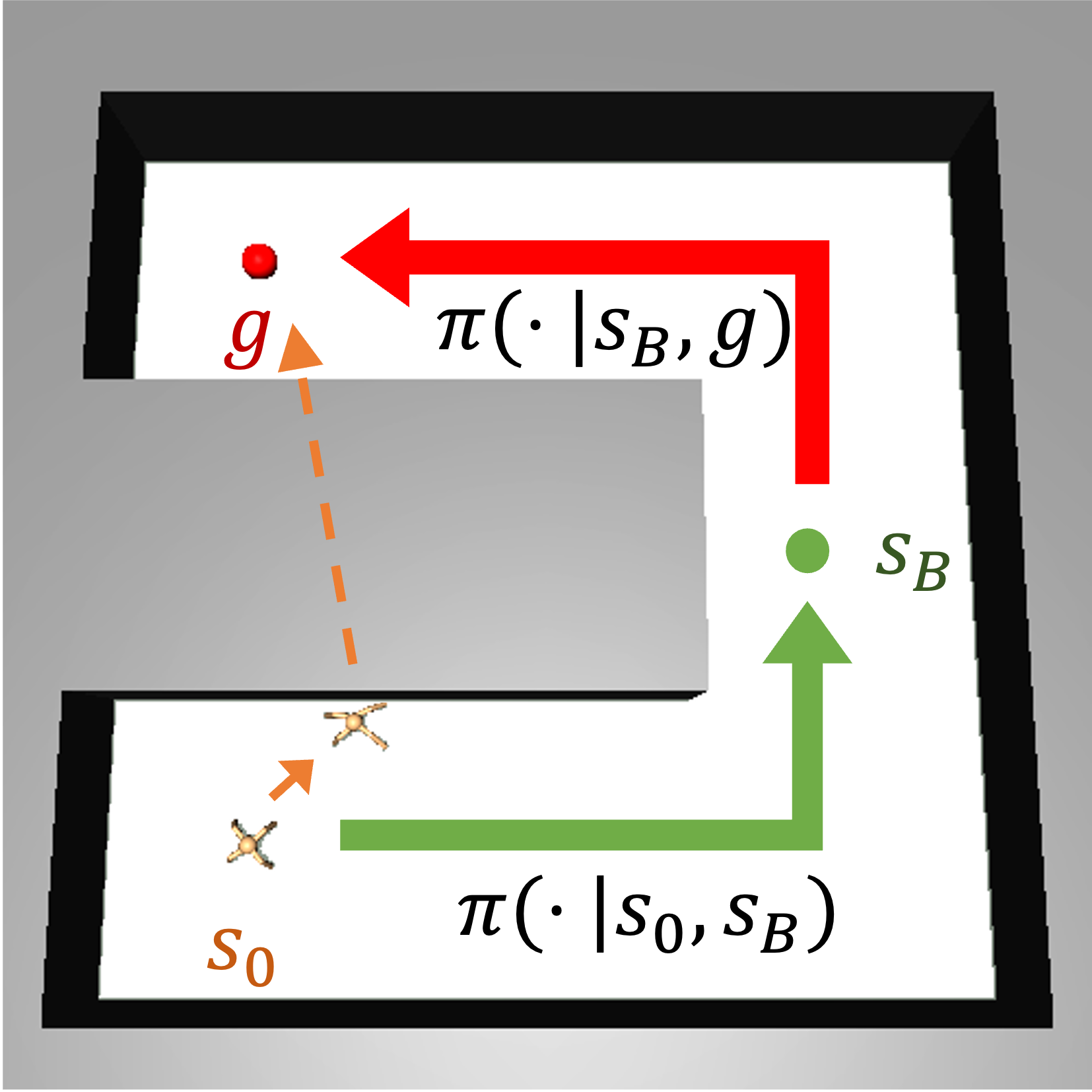}
    \includegraphics[width=0.49\linewidth]{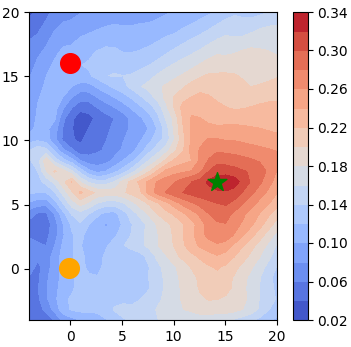}
    \caption{Illustration of task reduction. \emph{(Left)} Given a trajectory that fails to reach $g$ from $s_0$ (orange arrows), task reduction searches for an intermediate $s_B$, then executes the resulting two sub-tasks (green and red arrows) to generate a successful demonstration. \emph{(Right)} Heatmap of composed value in Eq.~\ref{eq:tr} w.r.t. $s_B$. The green star denotes the position of optimal $s_B$.}
    \label{fig:method:tr}
\end{figure}

%\yw{rewrite the following}
In order to perform effective supervised learning in the offline phase, it is critical that the online phase can generate as much successful trajectories as possible. 
In our framework, we consider two data augmentation techniques to boost the positive samples in the dataset $\mathcal{D}$, i.e., \emph{goal relabeling} and \emph{task reduction}. We will first describe the simpler one, goal relabeling, before moving to a much more powerful technique, task reduction.

\textbf{Goal relabeling } was originally proposed by \citet{andrychowicz2017hindsight}. In our goal-conditioned learning setting, for each failed trajectory $\tau=(g;s_t, a_t)$ originally targeted at goal $g$, we can create an artificial goal $g'$ by setting $g'=s_j$ for some reached state $s_j\in\tau$, which naturally yields a successful trajectory $\tau'$ as follows:
\begin{equation}
    \tau'\gets(g'=s_{j}; s_{t=0:j}, a_{t=0:j})\quad\textrm{where}\;\; s_j\in\tau.
\end{equation}
Therefore, despite its simplicity, goal relabeling can convert every failed trajectory $\tau$ to a positive demonstration without any further interactions with the environment. 

% after the trajectories are collected.
%Moreover, Note that goal relabeling does not require any further interactions with the environment. 

\textbf{Task reduction} was originally proposed by \citet{li2020solving}. The main idea is to decompose a challenging task into a composition of two simpler sub-tasks so that both sub-tasks can be solved by the current policy. 
As illustrated in the left part of Fig.~\ref{fig:method:tr}, given a goal $g$ from state $s_0$, task reduction searches for the best sub-goal $s_B^\star$ through a 1-step planning process over the universal value function $V_\psi(s,g)$ as follows
\begin{equation}\label{eq:tr}
    s_B^{\star} = \arg\max_{s_B} V_{\psi}(s_0, s_B) \oplus V_{\psi}(s_B, g),
\end{equation}
where $\oplus$ is a composition operator typically implemented as multiplication. Since the value function is learned, such a search process can be accomplished without any further environment interactions by gradient descent or cross-entropy method. 
A heatmap of the composed value for sub-goal search (Eq.(\ref{eq:tr})) is visualized in the right part of Fig.~\ref{fig:method:tr}.
Notably, after $s_B^\star$ is obtained, we would still need to execute the policy in the environment following the sub-goal $s_B^\star$ and then the final goal $g$ in order to obtain a valid demonstration.
Compared to goal relabeling, task reduction consumes additional samples and may even fail to produce a successful trajectory when either of the two sub-tasks fails. Hence, task reduction can be expensive in the early stage of training when the policy and value function have not yet been well trained. However, we will show both theoretically (Sec.~\ref{sec:theory}) and empirically (Sec.~\ref{sec:expr:ant}) that task reduction can lead to an exponentially faster convergence compared to using goal relabeling solely in long-horizon problems. 

By default, {\name} utilizes both goal relabeling and task reduction for data augmentation unless otherwise stated.

%We take task reduction as our default data augmentation technique unless otherwise stated. 
%\yw{unless the task is too simple for task reduction to take effect.}
%for some very simple task, where task reduction might be less sample efficient than goal relabeling due to extra sample consumption..

%task reduction performs task decompostion for goal $g$ from current state $s_0$ as a 1-step search process over the universal value function $V_\psi{}(s,g)$ by search for the best sub-goal $s_b$ to reach. 

%as an effective 1-step planning to tackle compositional problems. 

%In this work, we adopt a more powerful data augmentation technique, task reduction. 
%\yw{describe task reduction}

%\yw{discuss: task reduction consumes additional data; how to select threshold; probably a visualization here.}

\subsection{Offline SL Phase}\label{sec:offline}

%\yw{rewrite the following}

%\yw{talk about weighted BC; no value learning; }

After a dataset $\mathcal{D}$ of successful demonstrations, including augmented trajectories, is collected, we switch from RL training to offline SL by performing advantage weighted behavior cloning (BC). In particular, we set the weight $w(s,a,g)$ in Eq.~(\ref{eq:bc}) to $\exp{(\frac{1}{\beta}(R - V_{\phi}(s, g)))}$ following~\cite{peng2019advantage}. % \yw{no scaling coefficient?}
We remark that we only train policy in the offline phase while keeping the value function $V_{\psi}$ unchanged for algorithmic simplicity. We also empirically find that training value function during the offline phase does not improve the overall performance.

%We perform behavior cloning over the successful trajectories, including those augmented trajectories generated by task reduction. 
%We remark that we only train policy and do not perform additional value learning on the value function in the offline phase, which empirically does not help much. 

It is feasible to adopt more advanced methods in the offline phase, such as offline RL methods, which typically assume a pre-constructed dataset but conceptually compatible with our phasic learning process. We conduct experiments by substituting BC with two popular offline RL methods, decision transformer~\cite{chen2021decision} and AWAC~\cite{nair2020awac}, in Sec.~\ref{sec:expr:push}. Empirical results show that these alternatives are much more brittle than BC and perform poorly without a high-quality warm-start dataset. Thus, we simply use BC as our SL method in this paper.

%\yw{talk about offline RL alternative}
%Note that in the offline phase we can also adopt any other offline RL methods, such as AWAC and DT. Some recent offline RL methods also consider fine-tuning such as AWOPT. We conduct experiments in Sec. XXX which suggest that these offline RL methods can be brittle with self-generated data, which can be of particularly low quality in the initial stage of training. 

\subsection{{\name}: Phasic Self-Imitative Reduction}\label{sec:pair}
By repeatedly performing the online RL phase with task reduction and the offline SL phase, we derive our final algorithm, {\name}. The pseudo-code is summarized in Alg.~\ref{alg:main}.
More implementation details can be found in Appendix~\ref{app:pair_implement}.
\yf{More analysis on the phasic training scheme vs. joint RL and SL optimization is in Appendix~\ref{sec:app:phasic}.}
%\yw{also, do we have any additional implementation issues to mention?}

\begin{algorithm}[tb]
    \caption{Phasic Self-Imitative Reduction}
    \label{alg:main}
    \begin{algorithmic}[1]
        \STATE {\bf Initialize:} goal-conditioned policy $\pi_\theta(\cdot|s, g)$, value function $V_\psi(s, g)$.
            %Online and offline RL algorithms $\mathcal{A}_{\textrm{online}}$, $\mathcal{A}_{\textrm{offline}}$. 
        \FOR{$k \gets 1, 2, \cdots$}
            \STATE Sample a batch of trajectories $\mathcal{B}\gets\{\tau:(g;s_t,a_t)\}$ w.r.t. the current policy $\pi_\theta$
            \STATE Update $\pi_\theta, V_\psi$ by RL on $\mathcal{B}$ (Sec.~\ref{sec:rl-phase})
            \STATE Set dataset $\mathcal{D} \gets \emptyset$
            \FOR{$\tau \in \mathcal{B}$}
                \IF{\textbf{not} \texttt{is\_success}($\tau$)}
                    \STATE $\tau \gets \texttt{data\_augment}(\tau, \pi_\theta, V_\psi)$ (Sec.~\ref{sec:task-reduction})
                \ENDIF
                \STATE $\mathcal{D} \gets \mathcal{D} \cup\{\tau\}$
            \ENDFOR
            \STATE Train $\pi_\theta$ with SL over $\mathcal{D}$ (Sec.~\ref{sec:offline})%and (optional) $V$ 
            % \STATE Set dataset $\mathcal{D} \gets \emptyset$ and sample batch $\mathcal{B}\gets\emptyset$. % failure buffer $\mathcal{B}_{\textrm{fail}} \gets \emptyset$.
            % \FOR{$n \gets 1, 2, \cdots, N_{\textrm{online}}$}
            %     \STATE Randomly sample $s_0$ and $g$.
            %     \STATE Roll out a trajectory $\tau = (g; s_t, a_t)$ w.r.t. $\pi_\theta$.
            %     %,  $a_t \sim \pi(\cdot|s_t, g)$, for $t=0, 1, \cdots$.
            %     \STATE $\mathcal{B}\gets  \mathcal{B} \cup \tau$
            %     \IF{\textbf{not} \texttt{is\_success}($\tau$)}
            %         \STATE $\tau \gets \texttt{data\_augment}(\tau, \pi, V)$ (Sec.~\ref{sec:task-reduction})
            %     \ENDIF
            %     \STATE $\mathcal{D} \gets \mathcal{D} \cup\{\tau\}$
            %     % \ELSE
            %     %     \STATE $\mathcal{B} \gets  \mathcal{B} \cup \{\tau \}$
            %     % \ENDIF
            % \ENDFOR
            % \STATE Update $\pi_\theta, V_\psi$ by RL on $\mathcal{B}$.
            % % \FOR{$\tau \in \mathcal{B}_{\textrm{fail}}$}
            % %     \STATE $\tau' \gets \texttt{data\_augmentation}(\tau, \pi, V)$ (Sec.~\ref{sec:task-reduction})
            % %     \STATE $\mathcal{B} \gets \tau'$ \textbf{if} $is\_success(\tau')$
            % %     % \IF{$\tau'.r_T = 1$}
            % %     %     \STATE $\mathcal{B} \gets \tau' \cup \mathcal{B}$
            % %     % \ENDIF
            % % \ENDFOR
            % \STATE Train $\pi_\theta$ with SL over $\mathcal{D}$. %and (optional) $V$ 
        \ENDFOR
    \end{algorithmic}
\end{algorithm}
%\yw{***** theoretical justification *****}

%\yw{do we need a separate discussion section? (let's double check if all the positioning texts can be well explained in the related work section).}

%% file: 99theory.tex
%\jiaqi{Using the framework by Sergey and analyze the coefficient $\alpha$, show our $\alpha$}

%\begin{}

%\jiaqi{Doubling, Sergey needs O(depth) samples we need only log depth, for deterministic, and the goal is to learn pi( | goal=s) for every state s}

First, we establish the correctness of our framework. The following theorems follows directly from the GCSL framework in  \citep{ghosh2020learning}, because both GCSL and PAIR are built upon SL.

\begin{theorem}[\citet{ghosh2020learning}, Theorem 3.1] \label{thm:correct1} Let $J(\pi)$ be defined in Eq. (\ref{eq:rl}) and $J_{\mathrm{PAIR}}(\pi) = -L(\theta)$ as defined in Eq. (\ref{eq:bc}). Let $\tilde\pi$ be the data collection policy induced by data augmentation. Then  
	\begin{align}
		J(\pi) \ge J_{\mathrm{PAIR}}(\pi) - 4 T (T-1)\alpha^2 + C,
	\end{align}
where $\alpha = \max_{s, g} D_{\mathrm{TV}}(\pi(\cdot|s,g)\Vert\tilde\pi(\cdot|s,g))$ and $C$ is a constant independent of $\pi$.\footnote{$D_{\mathrm{TV}}(\mu \Vert \nu) := \frac12 \int_{\mathcal A} \lvert \mu(a)-\nu(a) \rvert~ \mathrm{d}{a}$ is the total variation between distribution $\mu$ and $\nu$ over the action set $\mathcal A$. }
\end{theorem}

\begin{theorem}[\citet{ghosh2020learning}, Theorem 3.2] \label{thm:correct2} Assume deterministic transition and that $\tilde\pi$ has full support.  Define 
	\begin{align}
		\epsilon:= \max_{s, g} D_{\mathrm{TV}}(\pi(\cdot|s,g)\Vert \hat\pi^\star(\cdot|s,g)).
	\end{align}
	Then $J(\pi^\star) - J(\pi) \le \epsilon \cdot T$, where $\hat \pi^\star$ minimizes $L(\theta)$ defined in Eq. (\ref{eq:bc}) and $\pi^\star$ is the optimal policy.
\end{theorem}

Here, we present theoretical justification to illustrate why our algorithm is  efficient on sparse-reward composite combinatorial tasks.

\begin{theorem} Under mild assumptions, the PAIR framework could use exponentially less number of iterations compared to the phasic framework that does not use task reduction, e.g., GCSL~\cite{ghosh2020learning}. %\yw{using goal relabeling?} \jiaqi{Yes.} 
\label{thm:exp}
\end{theorem}

We defer the exact statement of \cref{thm:exp} and its proof to \cref{app:proofexp}.

%% file: 50expr.tex
We aim to answer the following questions in this section:
\begin{itemize}
    \item \emph{Does phasic training of RL and SL objectives perform better than non-phasic joint optimization?}
    \item \emph{Is {\name} compatible with offline RL methods other than BC?}
    \item \emph{Does {\name} achieve exponential improvement over baselines \emph{without} task reduction?}
    \item \emph{Are all the algorithmic components of {\name} necessary for good performance?}
    \item \emph{Can {\name} be used to solve challenging long-horizon sparse-reward problems, e.g., cube stacking?}
\end{itemize}

%we consider 3 goal-conditioned continuous control problems with an increasing difficulty, including (i) short-horizon robotic pushing  (adopted from \cite{nair2018visual}), (ii) ant navigation in U-shaped maze, and (iii) robotic stacking with up to 6 cubes.

We consider 3 goal-conditioned control problems with increasing difficulty:
(i) short-horizon robotic pushing (adopted from~\cite{nair2018visual}), % in which a Sawyer arm aims to push a blue puck to a red goal spot (see Fig.~\ref{fig:expr:push_policy});
(ii) ant navigation in a U-shaped maze, % ``Ant Maze'' where the agent needs to control an ant robot to move to a goal position in a U-shape maze (Fig.~\ref{fig:expr:ant_policy});
and (iii) robotic stacking with up to 6 cubes. % robotic stacking (Fig.~\ref{fig:expr:stack_policy}), in which the target position of one cube is specified, and the agent gets reward only when the target cube remains at the goal position while the robot hand is far apart. 
All the tasks are with only 0-1 sparse reward.

We compare {\name} with non-phasic RL baslines that jointly perform RL and SL as well as phasic SL baselines that only perform SL over self-generated data.
For non-phasic RL baselines, we consider naive PPO, plain self-imitation learning with goal relabeling (SIL)~\cite{oh2018self} and self-imitation learning with task reduction (SIR)~\cite{li2020solving}.
For phasic SL baselines, we consider goal-conditioned supervised learning (GCSL)~\cite{ghosh2020learning}.
We emphasize that for a fair comparison, all the RL-based baselines leverage our proposed intrinsic rewards.
All the experiments are repeated over 3 random seeds on a single desktop machine with a GTX3090 GPU. More implementation and experiment details can be found in appendix.

%a collection of baseline methods, including non-phasic RL baselines which jo such as plain self-imitation learning and self-imitative reduction including plain self-imitation learning (SIL)~\yw{cite}

%We use PPO as the online RL algorithm, and advantage weighted imitation learning as the offline algorithm if not otherwise stated. All the experiments are repeated by 3 runs.

\input{51push}

\input{52antmaz}

\input{53stacking}

%% file: 51push.tex
\subsection{Sawyer Push}\label{sec:expr:push}
% \begin{figure}
%     \centering
%     \includegraphics[width=0.24\linewidth]{fig/sawyerpush/policy/tmp15.png}
%     \includegraphics[width=0.24\linewidth]{fig/sawyerpush/policy/tmp21.png}
%     \includegraphics[width=0.24\linewidth]{fig/sawyerpush/policy/tmp29.png}
%     \includegraphics[width=0.24\linewidth]{fig/sawyerpush/policy/tmp47.png}
%     \caption{Illustration of Sawyer Push environment and learned behavior.}
%     \label{fig:expr:push_policy}
% \end{figure}
\begin{figure}
    \begin{minipage}[t]{0.33\linewidth}
        \centering
        \includegraphics[width=0.9\linewidth]{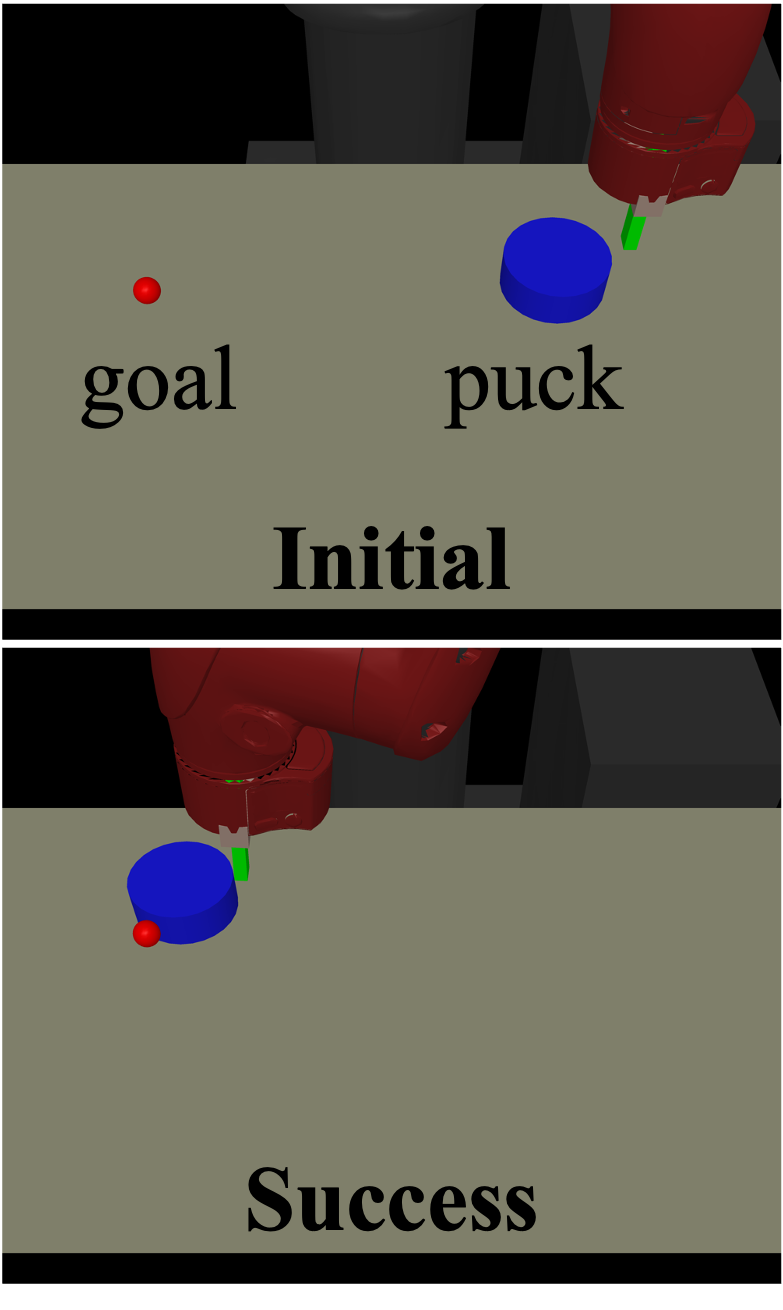}
        \caption{An initial state and a successful state in ``\textit{Push}'' environment.}
        \label{fig:expr:push_task}
    \end{minipage}
    \hspace{2mm}
    \begin{minipage}[t]{0.65\linewidth}
        \centering
        \includegraphics[width=\linewidth]{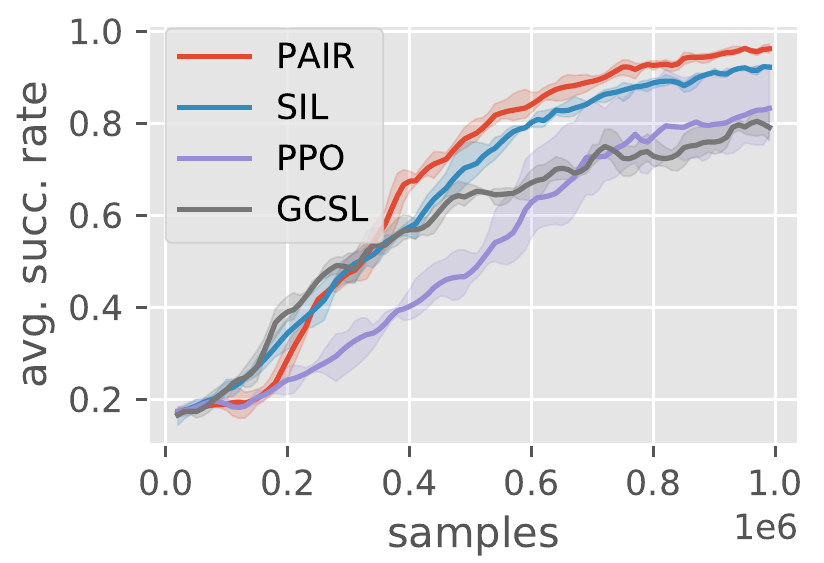}
        \caption{Average success rate vs. number of samples in ``\textit{Push}''. {\name} achieves the best performance compared to non-phasic, pure RL and pure SL baselines. }
    \label{fig:expr:sawyerpush}
    \end{minipage}
\end{figure}

\begin{figure}
    \centering
    \includegraphics[width=0.49\linewidth]{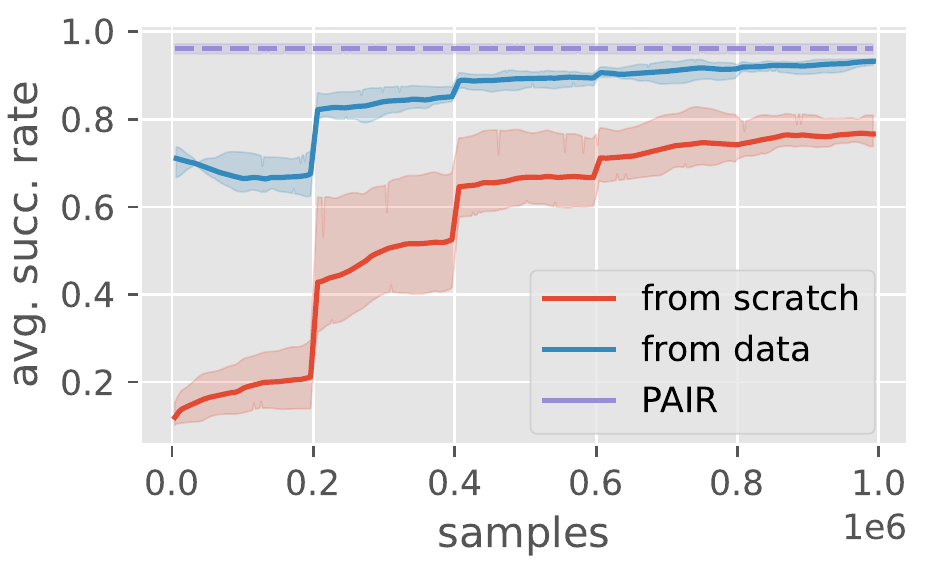}
    \includegraphics[width=0.49\linewidth]{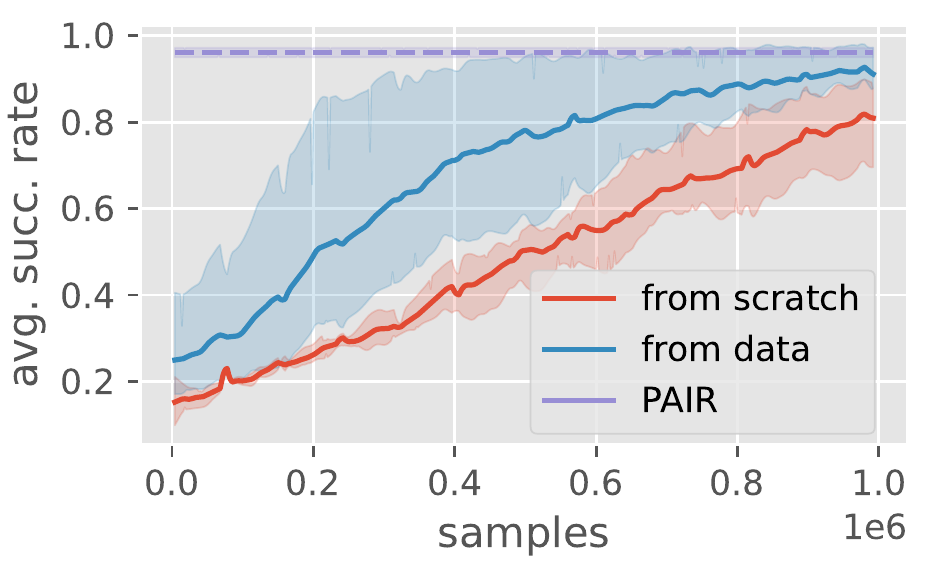}  
    \caption{Combining {\name} with offline algorithms AWAC (left) and DT (right) in ``\textit{Push}'' domain. We train them both from scratch (red) and from a demonstration dataset (blue). The performances of {\name} with PPO/BC are in dashed purple line.}
    \vspace{-3mm}
    \label{fig:expr:pair_offline}
\end{figure}

We first answer whether our phasic framework can outperform non-phasic training algorithms on a simple ``\textit{Push}'' task. As illustrated in Fig.~\ref{fig:expr:push_task}, a Sawyer robot is tasked to push the puck to the goal within 50 steps.  %\yw{mention the diameter! not just the horizon. We need to talk about the diameter and say this task diameter is too short so we do not use TR.}
The initial position of the robot hand, the puck and the goal are randomly initialized on the table. Since it only requires very few steps to reach the goal -- even the largest distance between puck and goal is within 20 steps of actions using a well trained policy, task reduction would not make its best use and may hurt sample efficiency due to extra sample consumption. In this particularly simple domain, we only use goal relabeling as a single data augmentation technique in {\name}.

\textbf{Effectiveness of phasic training:}
We compare {\name} with non-phasic RL method, i.e., SIL, which utilizes the successful relabeled demonstrations by jointly optimizing SL and RL objectives. As shown in Fig.~\ref{fig:expr:sawyerpush}, {\name} (red) gets higher success rate than SIL (blue) using fewer samples while both methods perform better than naive PPO, which is a pure RL baseline. We also compare with phasic SL method, GCSL~\cite{ghosh2020learning}, which only performs iterative SL on relabeled data without RL training. \yf{Following the original implementation of GCSL, every trajectory is relabeled as a successful one targeting at an achieved state.} GCSL learns fast in the beginning, but achieves a substantially lower final success rate than other methods.
%, therefore we do not consider GCSL in other more challenging domains. 
% We also visualize the learned policy of {\name} in Fig.~\ref{fig:expr:push_policy}. \yw{this illustration does not really meaningful... can you just change this to a task demo figure?}

\textbf{Combining with offline RL methods: }
Here we provide the results with initial attempts to combine our framework with representative offline RL methods, AWAC~\cite{nair2020awac} and decision transformer (DT)~\cite{chen2021decision}.

Original AWAC performs a single fine-tuning phase after offline pretraining~\cite{nair2020awac,lu2021awopt}. Following the phasic framework of {\name}, we alternate between offline and online AWAC updates. %The offline dataset in each offline phase consists of both collected and relabeled trajectories from all the previous online phases. 
We examine both training from scratch without any dataset prepared in advance, or starting from the offline phase with a warm-start dataset of successful trajectories. 
%We switch to offline training phase every $200k$ environment steps consumed by online training. 
The results are shown on the left of Fig.~\ref{fig:expr:pair_offline}. Phasic-AWAC from scratch (red) continues making progress as it switches from offline to online phase, but finally converges to a much worse policy compared with the variant initialized with a warm-start dataset (blue). %does not learn as well as the variant initialized with an offline dataset (blue).
% \yf{@Tian: talk abount DT results here.}

DT predicts actions conditioning on a sequence of desired returns, past states and actions via a transformer model~\cite{vaswani2017attention}.  Since DT is proposed only for a single offline phase, we similarly adopt PPO for RL training in the online phase. We use a context length of 5 for sequence conditioning and train DT both from scratch and from prepared warm-start dataset with successful demonstrations. 
%We switch DT from online to offline phase every $80k$ environment steps consumed. 
As shown in the right plot of Fig.~\ref{fig:expr:pair_offline}, the performance progressively improves within each online phase and SL phase. % leads to a small gain during each offline phase along with a slight decline but later higher increasing rate in the following online phase.  
Phasic-DT with warm-start (blue) outperforms the variant from scratch and gets similar final performance as {\name} but the variance is much higher. % compared to {\name}is with high variance.
% Phasic DT from scratch (red) learns as well as {\name} (purple) and has a performance comparable with phasic DT bootstrapped from successful demonstrations (blue) after the convergence. \yw{update}

\begin{figure}
    \centering
    \includegraphics[width=0.49\linewidth]{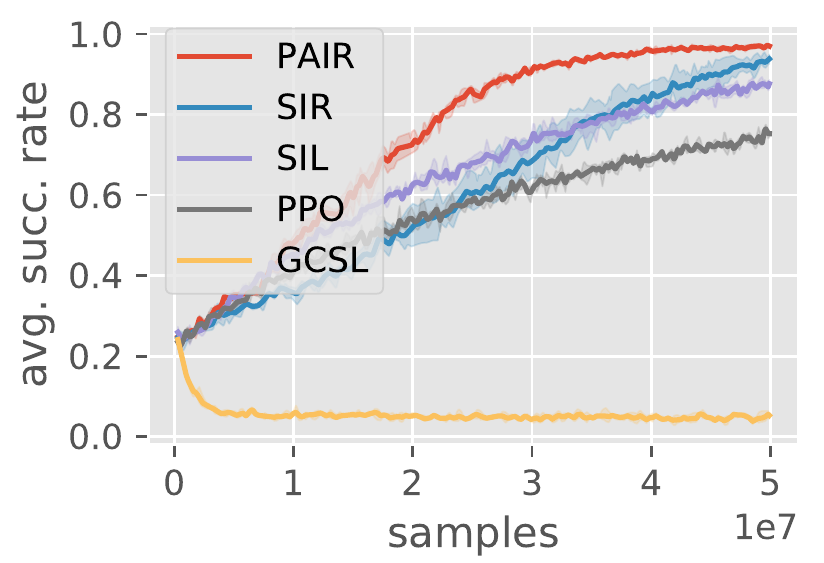}
    \includegraphics[width=0.49\linewidth]{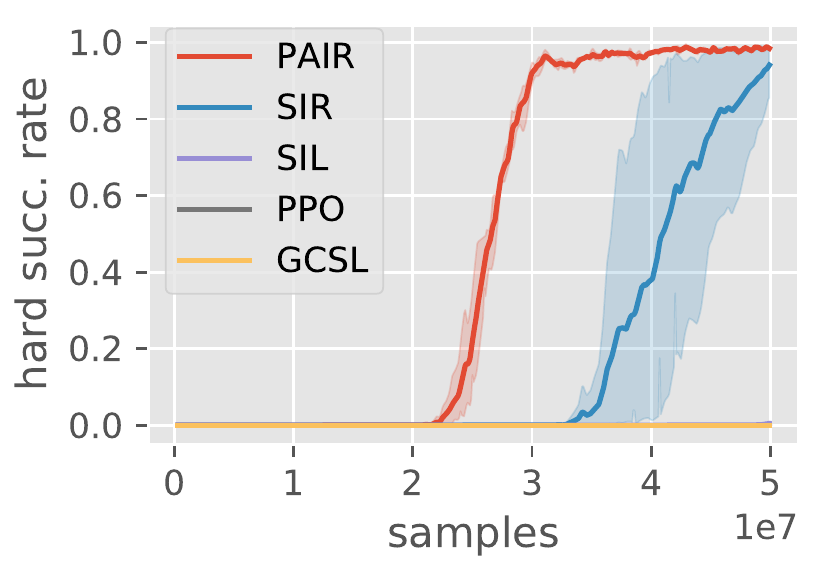}
    \caption{(\textit{Left}) Average success rate over uniformly sampled tasks, where ant and goal positions are uniformly sampled. (\textit{Right}) Success rate evaluated on one particularly difficult task configuration with ant and goal initialized at two ends of the maze.}
    \vspace{-3mm}
    \label{fig:expr:antmaze}
\end{figure}
\begin{figure*}
    \centering
    \begin{subfigure}{0.3\textwidth}
        \includegraphics[width=\linewidth]{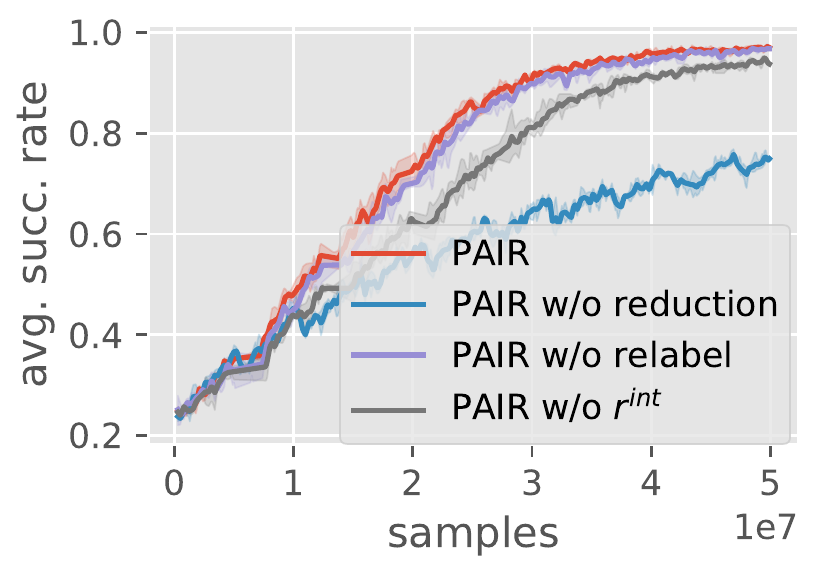}
        \caption{Ablate {\name} components.}\label{fig:expr:antmaze_ablation_main}
    \end{subfigure}
    \begin{subfigure}{0.3\textwidth}
        \includegraphics[width=\linewidth]{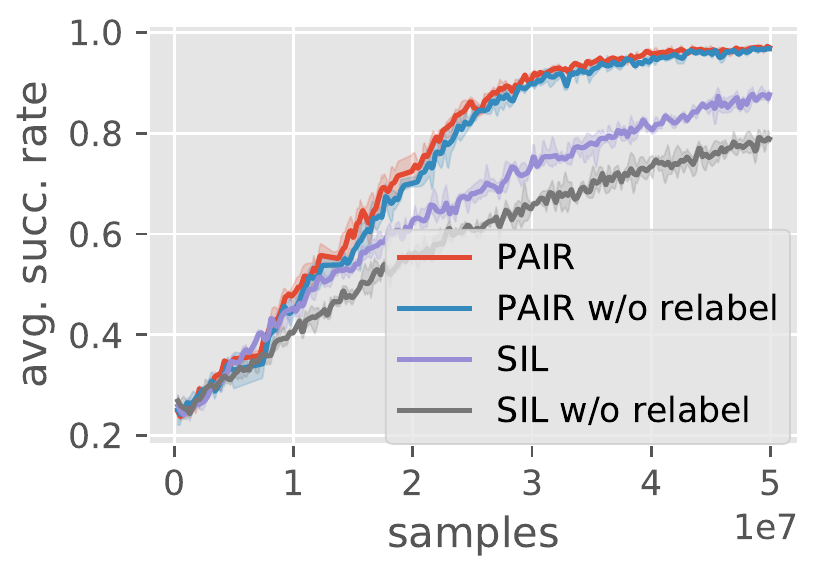}
        \caption{The effect of goal-relabeling on SIL.}\label{fig:expr:antmaze_ablation_relabel}
    \end{subfigure}
    \begin{subfigure}{0.3\textwidth}
        \includegraphics[width=\linewidth]{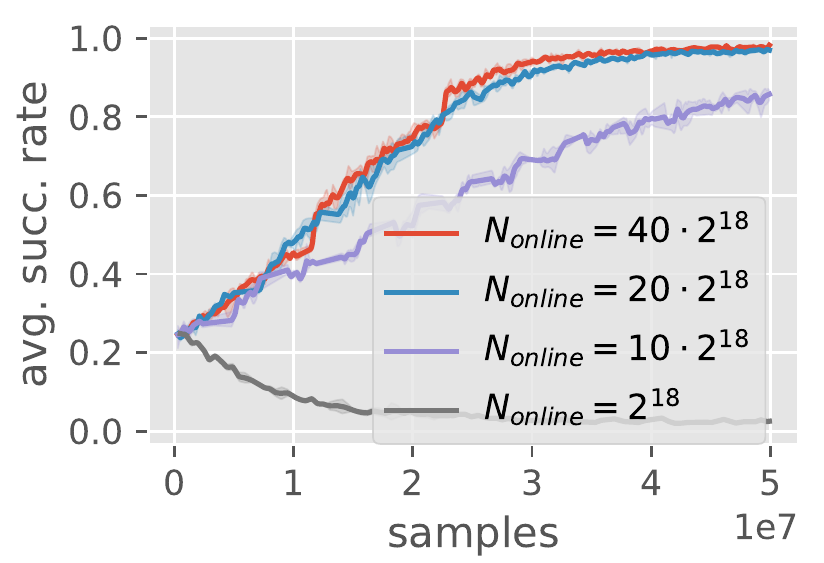}
        \caption{Sensitivity analysis of $N_{\textrm{online}}$.}\label{fig:expr:antmaze_ablation_n}
    \end{subfigure}
    \caption{Ablation studies on algorithmic components of {\name} in ``\textit{Ant Maze}''. (a) Performances after removing different components of {\name}. Using reduction data and intrinsic reward help the most for {\name}. (b) Validate the effectiveness of goal relabeling on SIL. (c) Comparison of different $N_{\textrm{online}}$, i.e., samples collected in the RL phase. {\name} works well with less frequent phase switches.}
    \label{fig:expr:antmaze_ablation}
\end{figure*}

We generally observe that offline RL algorithms within our phasic framework would require a good dataset to initialize and do not perform as robustly as the simple PPO and BC combination when starting from scratch. Therefore, we confirm our use of BC/PPO for the remaining experiments and leave a more competitive combination with offline RL algorithms as future work.

% \begin{figure}
%     \begin{minipage}{0.49\linewidth}
%         \centering
%         \includegraphics[width=\linewidth]{fig/sawyerpush/SawyerPush-v1_awac.pdf}
%         \caption{Combine {\name} with AWAC.}
%         \label{fig:expr:pair_awac}
%     \end{minipage}
%     \begin{minipage}{0.49\linewidth}
%         \centering
%         \includegraphics[width=\linewidth]{fig/sawyerpush/SawyerPush-v1_dt.pdf}
%         \caption{Combine {\name} with DT.\tian{Update from data curve.}\yf{add PAIR curve}}
%         \label{fig:expr:pair_dt}
%     \end{minipage}
% \end{figure}

% \begin{figure}
%     \centering
%     \includegraphics[width=0.49\linewidth]{fig/sawyerpush/SawyerPush-v1_awac.pdf}
%     \caption{Combine {\name} with AWAC.}
%     \label{fig:expr:pair_awac}
% \end{figure}

% \begin{figure}
    
% \end{figure}

%% file: 52antmaz.tex
\subsection{Ant Maze}\label{sec:expr:ant}

We then consider a harder task ``\emph{Ant Maze}'', which requires a locomotion policy with continuous actions to control the ant robot and plan over an extended horizon to navigate inside a 2-D U-shape maze. The observations include the center of mass and the joint position and velocities of the ant robot. The goal is a 2-D vector denoting the $xy$ position of that the robot should navigate to. 
For each episode, the initial position of the ant's center and the goal position are uniformly sampled over the empty space inside the maze. We define the \emph{hard task} configuration as when the ant and the goal are initialized at two different ends of the maze (see Fig.~\ref{fig:method:tr}).
In this domain, we leverage both task reduction and goal relabeling as data augmentation techniques in {\name}.

\textbf{Effectiveness of {\name} and exponential improvement on hard goals:}
%We employ task reduction in addition to goal relabeling for data augmentation in this task. 
As shown in Fig.~\ref{fig:expr:antmaze}, we evaluate the performances of different algorithms with the average success rate over the entire task space (left) and on hard tasks only (right). We compare {\name} with: SIR, which runs SL and RL jointly using both reduction and relabeling data; SIL, a plain non-phasic RL method without task reduction; GCSL and vanilla PPO. We observe that all the methods that utilize task reduction (i.e., {\name}, SIR) can finally converge to a much higher average success rate (left plot) and outperform SIL and vanilla PPO. 
The gap becomes substantially larger on the performance on hard situations (right plot), where only methods with task reduction are able to produce non-zero success rate within the given sample budget. 
This empirical observation is consistent with our theoretical analysis in Sec.~\ref{sec:theory} on the effectiveness of task reduction. 
We also remark that {\name} achieves a significantly higher sample efficiency compared with the non-phasic method SIR on the hard cases.
Notably, GCSL completely fails in this problem. We empirically observe that the ant robot may frequently get stuck in the early stage of training (e.g., it may accidentally fall over and cannot recover anymore). We hypothesis that the failure of GCSL is largely due to a tremendous amount of supervision from such corrupted trajectories. This suggests that online RL training would be necessary compared with running SL only (e.g., SIL performs well on this task).

%\yf{GCSL fails in this domain since its policy can be easily corrupted by harmful relabeled data generated from bad rollouts that get stuck at some point (e.g., when the ant accidentally falls over and cannot walk anymore), while RL tuning can overcome this issue more easily.}  

%This empirical observation is consistent with our theoretical analysis in Sec.~\ref{sec:theory} on the effectiveness of task reduction. 
%We also remark that {\name} achieves a significantly higher sample efficiency compared with the non-phasic method SIR on the hard cases.

%Notably, SIL that only utilizes relabeling data achieve 0 success rate on hard tasks. Comparing {\name} with SIR, we can see that phasic framework can further boost the sample efficiency than non-phasic joint optimization. PPO that does not use augmented data performs even worse. GCSL in this domain fails to learn, indicating that online RL training is critical for challenging goal reaching problems.

% The learned behavior on the hard task is depicted in Fig.~\ref{fig:expr:ant_policy}.
\begin{figure}
    \centering
    \includegraphics[width=\linewidth]{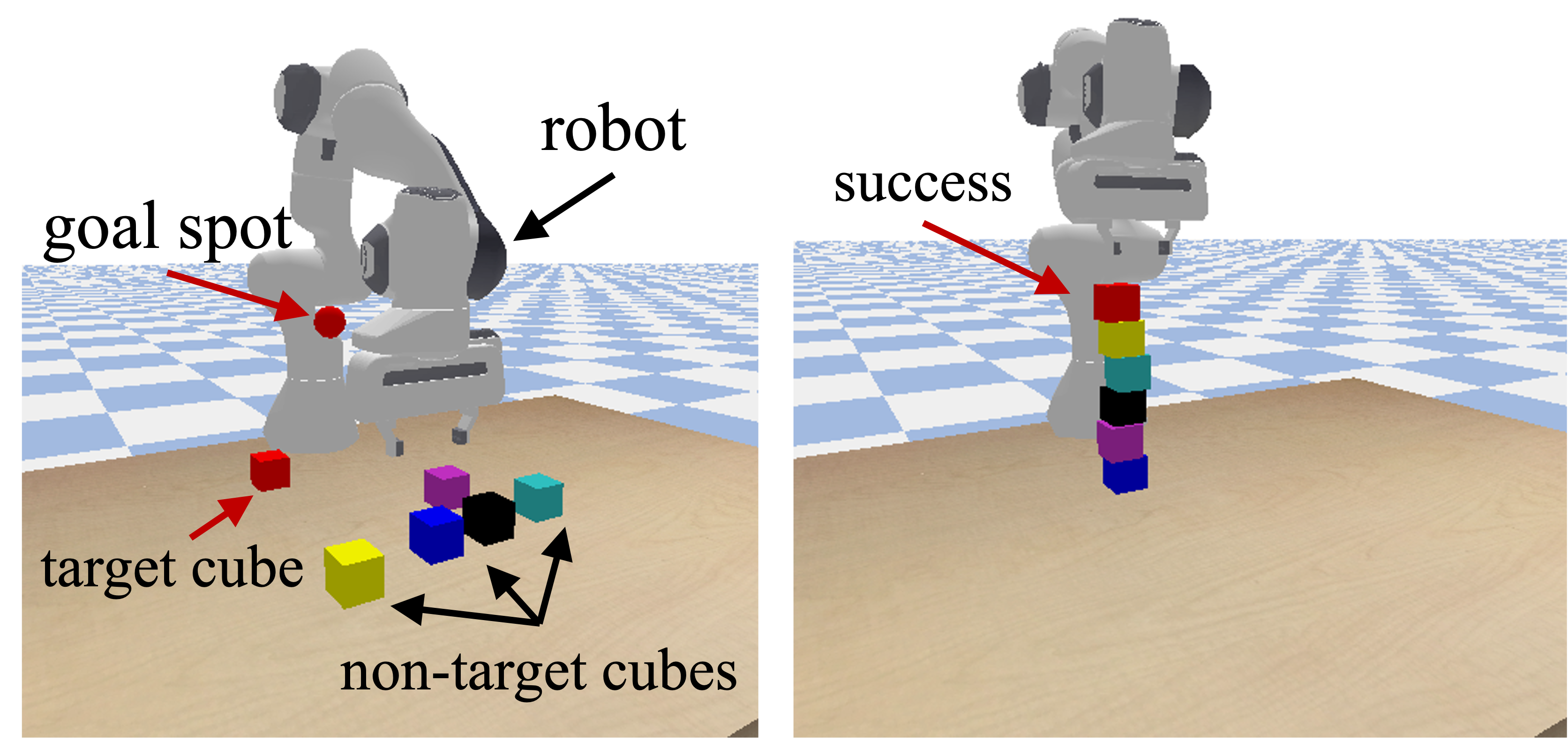}
    \caption{Illustration of the stacking environment. The left figure shows an initial state, where the red ball denotes the identity and the desired position of a target cube. On the right we show a successful state for this task, where the red cube is close to the goal and the robot hand does not touch the tower.}
    \vspace{-3mm}
    \label{fig:expr:stack_task}
\end{figure}

\begin{figure*}
    % \centering
    \begin{minipage}{0.48\textwidth}
        \centering
        \includegraphics[width=0.7\linewidth]{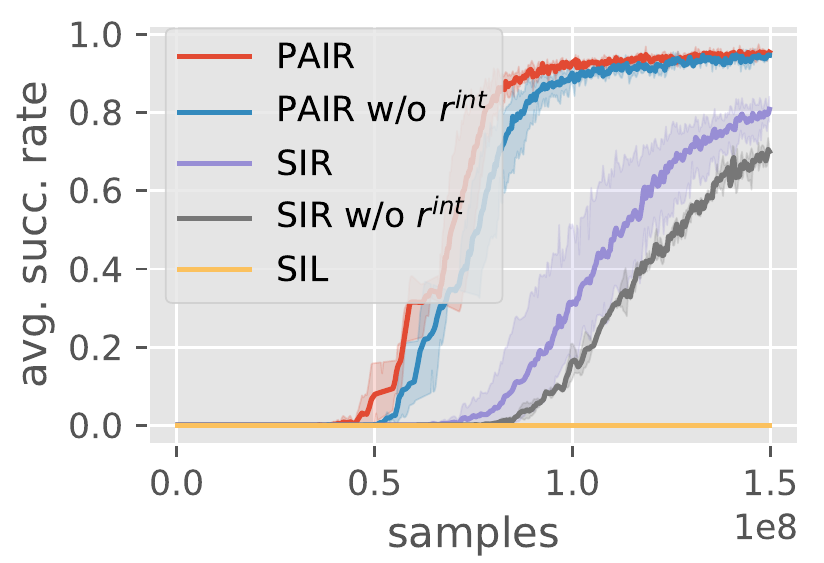}
        \caption{Average success rate in ``stack-6-cube'' with sparse reward. {\name} solves the task with high success rate most efficiently.}
        \label{fig:expr:stacking}
    \end{minipage}
    \hspace{2mm}
    \begin{minipage}{0.49\linewidth}
        \centering
        \begin{tabular}{lcc}
        \toprule
            Algorithm & SR @ 7.5e7 steps & SR @ 1.5e8 steps\\
        \midrule
            {\name} & \textbf{0.693 $\pm$ 0.176} & \textbf{0.955 $\pm$ 0.005}\\
            {\name} w/o $r^{\textrm{int}}$ & 0.493 $\pm$ 0.156 & 0.952 $\pm$ 0.016\\
            SIR & 0.026 $\pm$ 0.043 & 0.815 $\pm$ 0.039\\
            SIR w/o $r^{\textrm{int}}$ & 0.003 $\pm$ 0.003 & 0.688 $\pm$ 0.011\\
            SIL & 0.000 $\pm$ 0.000 & 0.000 $\pm$ 0.000\\
        \bottomrule
        \end{tabular}
    \captionsetup{type=table}\caption{Mean and standard deviation of success rates for stacking 6 boxes with different algorithms over 3 seeds. The policies are evaluated at 7.5e7 and 1.5e8 environment samples.}
    \label{tab:expr:stack_sr}
    \end{minipage}
    % \begin{minipage}{0.39\linewidth}
    %     \centering
    %     \includegraphics[width=\linewidth]{fig/stacking/stack_task.png}
    %     \caption{}
    %     \label{fig:expr:stack_task}
    % \end{minipage}
    % \hspace{1mm}
        
\end{figure*}
\begin{figure*}[tb]
    \centering
    \includegraphics[width=0.15\linewidth]{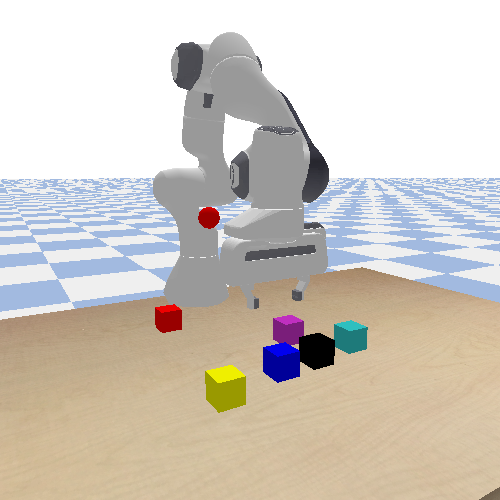}
    \includegraphics[width=0.15\linewidth]{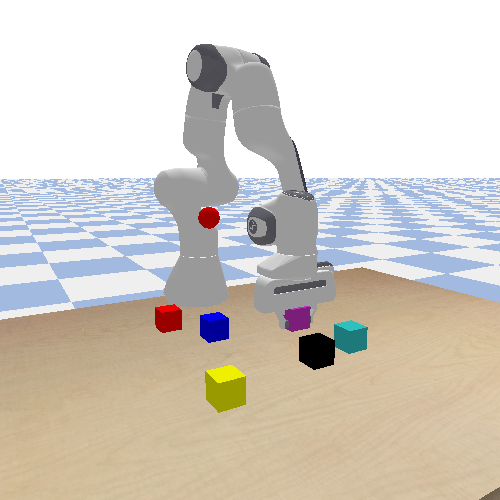}
    \includegraphics[width=0.15\linewidth]{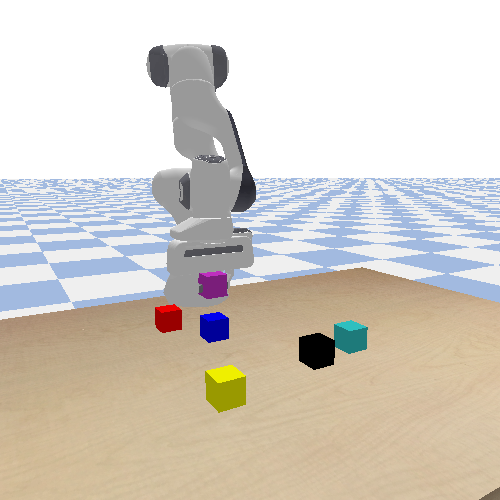}
    \includegraphics[width=0.15\linewidth]{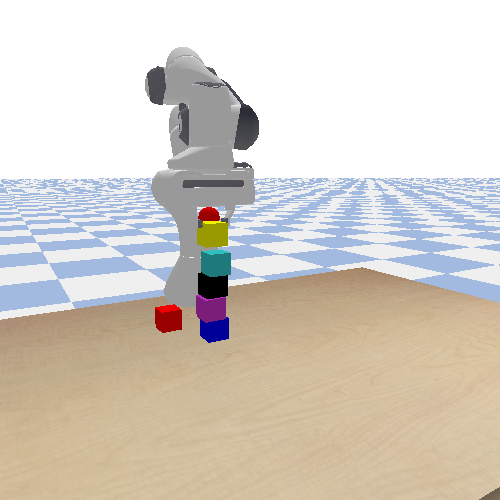}
    \includegraphics[width=0.15\linewidth]{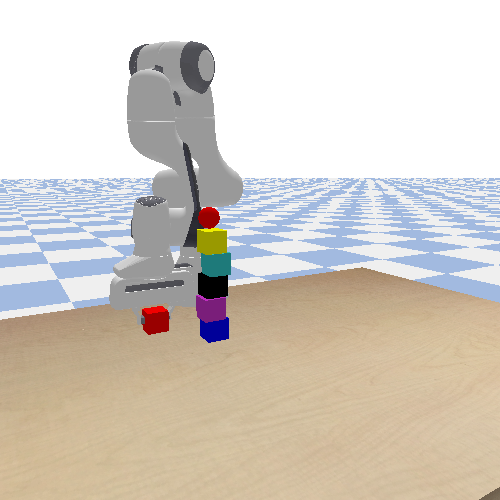}
    \includegraphics[width=0.15\linewidth]{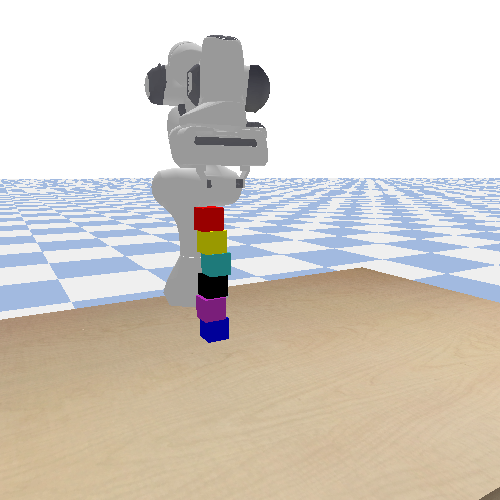}
    \vspace{-1mm}
    \caption{Learned policy of {\name} agent for stacking 6 randomly initialized cubes. }
    \label{fig:expr:stack_policy}
\end{figure*}

\textbf{Ablation studies:}
We then study the effectiveness of different components of {\name}. From Figure~\ref{fig:expr:antmaze_ablation_main}, we can find that removing value-difference-based intrinsic reward makes {\name} converge slower while the policy is still able to achieve a high success rate after convergence. After removing task reduction, the performance becomes significantly worse. By contrast, when removing goal relabeling, the performance drop is negligible. These observations suggest that task reduction is the most critical component in {\name}.
%Using relabeled data or not has marginal effect on the training process, possibly because they already use more powerful task reduction data. 
In Figure~\ref{fig:expr:antmaze_ablation_relabel}, we additionally examine the effectiveness of goal relabeling within a non-phasic RL framework, i.e., SIL, where we can clearly observe a performance drop of SIL with goal relabeling turned off. This indicates that goal relabeling can be still beneficial for methods that do not utilize task reduction.
Finally, in Figure~\ref{fig:expr:antmaze_ablation_n}, we also perform sensitivity analysis on the frequency of alternating between the offline and online phase. We perform one offline SL phase after collecting $N_{\textrm{online}}$ samples in the online RL phase.
We can observe that a relatively low alternating frequency is critical to the success of {\name}, which suggests that a large dataset for SL training and a long RL fine-tuning period help accelerate learning. We notice that when the phase changes too frequently, {\name} can be unstable or even fail to learn.

%% file: 53stacking.tex
\subsection{Stacking}

% \begin{table}[]
%     \centering
    
% \end{table}

Finally, we test whether {\name} can solve an extremely challenging long-horizon sparse-reward robotic manipulation problem. 
% \yw{let's try to make bold texts consistent and ensure bold tests are only for question answer. So let's try to shorten the task description and remove the bold texts.}
% \textbf{Task definition and training specifications: }
We build a robotic control environment that features stacking multiple cubes given only final success reward.
The simulated scene is shown in Fig.~\ref{fig:expr:stack_task}. A Franka Panda arm is mounted on the side of a table. On top of the table, there are a total of $N$ cubes with random initial positions. A goal spot specifies a target cube using its color and also its desired position. The robot's task is to manipulate the cubes on the table so that the specified target cube can remain stable within 3cm distance to the goal position, while its hand is at least 10cm apart from that position at the same time. In order to accomplish this task, the robot must build a ``base'' tower using other non-target cubes. During the whole construction process, \emph{the robot cannot receive any external reward for intermediate manipulations unless when the target cube is in place}. We perform curriculum on the number of cubes to stack, and use the similar task reduction process described in~\cite{li2020solving}. More details of training specifications can be found in Appendix~\ref{app:stack_train_detail}.

\textbf{Performance on stacking with 6 cubes: }
We report the success rate on the most challenging ``stack-6-cube'' scenario using {\name}, SIR and SIL algorithms in Fig.~\ref{fig:expr:stacking} and Table~\ref{tab:expr:stack_sr}. {\name} achieves an impressive 95.5\% success rate in this challenging task. SIR is making considerable progress thanks to task reduction, but it learns significantly less sample efficiently than {\name}. SIL baseline without task reduction fails completely, getting 0 success rate throughout the training process. % which suggests that successful data with strong compositionality generated by task reduction are critical to solve such long horizon sparser reward problems. 
We also show the effectiveness of intrinsic reward in online training: similar to the findings in other domains, $r^{\textrm{int}}$ can significantly speed up the training process.  

% \yw{say a bit more?}
% \yf{Punchline: stacking. Our method can solve extremely challenging sparse reward ``stack 6'' task while (simultaneous SIR, PPO + SIL, naive PPO) all fail.}

\textbf{Learned strategies: }
% show the learned policy
The learned policy using {\name} is visualized in Fig.~\ref{fig:expr:stack_policy}. In the initial frame, all 6 cubes are scattered randomly on the table. The robot then picks up non-target cubes, transports them to accurate positions aligned with the goal spot one by one. Even when picking up the red cube which locates close to the half-built tower, the robot is cautious enough to avoid knocking down the tower.
% Although the target red cube is located very close to the half-done tower, the robot pays enough attention to this situation, and successfully gets around the existing tower without any detrimental contact with it. 

%% file: 60conclusion.tex
We propose a phasic training method {\name} that efficiently combines offline supervised learning with online reinforcement learning for sparse-reward goal-conditioned problems, such as robotic stacking of multiple cubes. %Different from other works that incorporates SL into RL, 
{\name} repeatedly alternates between SL on self-generated datasets and RL fine-tuning and leverages value difference as intrinsic rewards and task reduction as data augmentation.
We validate the effectiveness of {\name} both theoretically and empirically on a variety of domains. 
We remark that {\name} provides a general learning paradigm that has the potential to be combined with more advanced offline RL methods, even though our initial attempts are not satisfactory.
%
%despite our currently its performance is limited by the requirement on a good initial dataset to warm up. 
%
We hope {\name} can be a promising step to take advantage of both supervised learning and reinforcement learning and help make RL a more scalable tool for complex real-world challenges.

%. We apply task reduction as a data augmentation technique to generate data of higher quality for challenging sparse reward problems, and justify its effectiveness both theoretically and empirically. {\name} has the potential to be combined with more advanced offline RL methods, although currently its performance is limited by the requirement on a good initial dataset to warm up. We hope {\name} can be a promising step to take advantage of both supervised learning and reinforcement learning and help make RL a more scalable tool for complex real-world challenges.

%% file: 90theoryproof.tex
\label{app:proofexp}

In this section, we prove \cref{thm:exp}.

Here, we make the following assumptions to facilitate our analysis. First, we consider a goal-conditioned MDP with deterministic transition. We assume that the goal set is the same as the state space, i.e., $\mathcal G = \mathcal S$. Furthermore, we consider sparse-reward environment where the reward function is defined by $r(s, a, g) = \mathbb I\{s=g\}$. 

We note that our framework draw on-policy trajectories with random state-goal pairs. To further simplify our theoretical analysis, we assume that each state-goal pair is sampled at least once in each iteration. Here, we may meet two practical issues. First, the state space may be continuous, which makes it impossible to sample each state-goal pair once. This may fit into our theoretical analysis by discretizing the state space. Second, $N_{\mathrm{online}}$ may not be large enough. This fits into our analysis by merging several consecutive iterations. 

We use $\pi^{(k)}$ to denote the policy by the end of the $k$-th iteration, and $\pi^{(0)}$ denotes the initial policy. We assume that the initial policy $\pi^{(0)}$ can reach any one-step goal. Specifically, if $P(s'|s, a)=1$ then $\pi^{(0)}(a | s, g=s')=1$. 

For a state $s$, a goal $g$, and a (deterministic) policy $\pi$, we define 
\begin{align}
	s \xrightarrow{\pi} g := \begin{cases}
		\text{true}, &\Pr[\exists t \ge 0 : s_t = g \mid s_0 = s, \forall i \ge 0, a_i \sim \pi(\cdot | s_i, g), s_{i+1} \sim P(\cdot | s_i, a_i)] = 1, \\
		\text{false}, & \text{otherwise}.
		\end{cases}
\end{align}
and we define 
\begin{align}
	d(s, s') &:= \min \{t : \exists \mathrm{ policy } \pi \text{ such that } \Pr[s_t = s' \mid s_0 = s, \forall i \ge 0, a_i \sim \pi(\cdot | s_i, g), s_{i+1} \sim P(\cdot | s_i, a_i)] =1\}, \\ 
	\ell_k &:= \sup_{\ell} \{\forall s, g : d(s, g) \le \ell : s \xrightarrow{\pi^{(k)}} g \text{ is true}\}.
\end{align}
Finally, we let 
\begin{align}
	D = \max_{s, s' \in \mathcal S} \{d(s, s') : d(s, s') < +\infty\} \label{eq:diameter}
\end{align}
be the maximum possible length of trajectory from state $s$ to goal $s'$ such that $s'$ is reachable from $s$.

\begin{lemma} \label{lem:reduc-1} For the PAIR framework,  we have $\ell_k \ge 2 \ell_{k-1}$ for every iteration $k$. 
\end{lemma}

\begin{proof} We prove by induction. For every state-goal pair $s, g$ such that $d(s, g) \le 2\ell_{k-1}$, we can find some state $s'$ such that $d(s, s') \le \ell_{k-1}$ and $d(s', g) \le \ell_{k-1}$. By the inductive hypothesis, we have that $s \xrightarrow{\pi^{(k)}} s'$ and $s' \xrightarrow{\pi^{(k)}} g$ are true. 
	
	Now we claim that every $(s, g) \in \mathcal S \times \mathcal G$ would have a success trajectory after the $k$-th iteration. Note that by our assumption, the framework would roll out a trajectory from $s$ to $g$ using $\pi(\cdot|s, g)$. If it succeeds, then we get a trajectory, else task reduction (cf. \cref{sec:task-reduction}) would find the state $s'$ such that $s \xrightarrow{\pi^{(k-1)}} s'$ and $s' \xrightarrow{\pi^{(k-1)}} g$ are true, because $V(s, s') = V(s', g) = 1$. Therefore, task reduction would run $\pi^{(k-1)}(\cdot|s, g=s')$ followed by $\pi^{(k-1)}(\cdot|s', g)$,  and get a successful trajectory from $s$ to $g$. Finally, the SL step would learn the policy $\pi^{(k)}$ such that $s\xrightarrow{\pi^{(k)}}g$ is true.  
\end{proof}

\begin{lemma} \label{lem:reduc-2} The PAIR framework uses at most $O(\lvert \mathcal S \rvert^2 \log D )$ samples to learn a policy that could go from any state to any reachable goal. 
\end{lemma}

\begin{proof} By the definition of $\ell_k$, we know that $\pi^{(k)}$ could go from any state to any reachable goal if $\ell_k \ge D$. By our assumption, we have $\ell_0 = 1$. Therefore, by \cref{lem:reduc-1}, we have $\ell_k \ge D$ after $k = O(\log D)$, i.e., PAIR uses at most $O(\log D)$ iterations.
	
	By our assumption, PAIR uses $O( \lvert \mathcal S \rvert \lvert \mathcal G \rvert) = O(\lvert \mathcal S \rvert^2)$ samples, because each state-goal pair is sampled for constant number of times. Thus we conclude that PAIR uses at most $O(\lvert \mathcal S \rvert^2 \log D )$ samples in total.
\end{proof}

\begin{lemma} Without PAIR, under assumptions described in \cref{app:proofexp}, for hard tasks, with high probability, the algorithm needs $\Omega(\lvert \mathcal S \rvert^2 D)$ samples to learn a policy that could go from any state to any reachable goal. \label{lem:reduc-3}
\end{lemma}

\begin{proof} We consider the following hard task. Suppose $s_0, a_0, \cdots, s_D$ is a trajectory of length $D$ that maximizes Eq. (\ref{eq:diameter}). We assume that the initial policy $\pi^{(0)}$ is a uniformly random policy and we assume $\Pr[\pi^{(0)}(\cdot | s_i, g=s_D)] ]\propto \lvert \mathcal A \rvert^{-1}$ for $i \le D-2$.\footnote{Actually, we can construct such an MDP. Assume we have $(D+2)$ states $\mathcal S = \{s_0, \cdots, s_D, s_{D+1}\}$ and $2$ actions $\mathcal A = \{0, 1\}$. The transition is given by $P(s_{D+1} | s_i, a=0) = 1$ and $P(s_{i+1} | s_i, a=1) = 1$ for $0 \le i \le D$. Such an MDP can similarly be constructed for $\lvert \mathcal A \rvert \ge 3$. } We observe that for the SL framework, if $s_i \xrightarrow{\pi^{(k)}} s_D$ is true then the dataset must contain a trajectory from $s_i$ to $s_D$, and thus the policy $\pi^{(k)}$ learned by SL would satisfy $s_{j} \xrightarrow{\pi^{(k)}} s_D$ for every $j \ge i$. For this, we analyze the sample complexity by analyzing how many iterations would be enough for the algorithm to learn a policy from state $s_0$ to goal $s_D$. To facilitate analysis, we define $i^{(k)}$ to be the smallest number such that $s_{i^{(k)}} \xrightarrow{\pi^{(k)}} s_D$. Then the algorithm learns the desired policy after $k$ iterations only if $i^(k) =0$. 
	
Next, we prove that $\mathbb E[i^{(k-1)} - i^{(k)} | i^{(k-1)}] \le O(1)$. To change $i^{(k-1)}$, the algorithm must roll out a success trajectory with goal $s_D$ from some $s_j$ with $j \le i$. Note that 
\begin{align}
	\Pr[i^{(k-1)} - i^{(k)} = j | i^{(k-1)}] &\le \Pr[\text{algorithm rolls out a success trajectory from }s_{i^{(k-1)}-j} \text{ to } s_D] \\
	&\le \Pr[\pi^{(k-1)} \text{ could go from }s_{i^{(k-1)}-j}\text{ to }s_{i^{(k-1)}}] \\ 
	&\le O(\lvert \mathcal A \rvert^{-j}),
\end{align}
so 
\begin{align}
	\mathbb E[i^{(k-1)} - i^{(k)} | i^{(k-1)}] &= \sum_{j=0}^{i^{(k-1)}} j\cdot \Pr[i^{(k-1)} - i^{(k)} = j | i^{(k-1)}] \\ 
	&\le O( \sum_{j = 0}^\infty j \cdot \lvert \mathcal A \rvert^{-j}) \\ 
	&\le O(\lvert \mathcal A \rvert^{-1}) \\ 
	&\le O(1).
\end{align}
Therefore, by the Azuma-Hoeffding inequality \citep{azuma1967weighted}, we conclude that with high probability $1-\delta$, we have $i^{(k)} < D$ for $k = \Theta(D - \sqrt{D \log \delta^{-1}}) \ge \Omega(D)$, i.e., $k \ge \Omega(D)$ with high probability. Together with that $\lvert \mathcal S\rvert^2$ trajectories are rolled out in each iteration, we conclude that the sample complexity is with high probability at least $\Omega(\lvert \mathcal S \rvert^2 D)$. 
\end{proof}

Finally, we conclude \cref{thm:exp} by collecting Lemmas \ref{lem:reduc-2} and \ref{lem:reduc-3}.

%% file: 80expr_detail.tex
\section{Experiment Details}
\subsection{Environment Description}
In this section, we describe all the environment-specific details regarding MDP definitions.
\subsubsection{Sawyer Push}
``\textit{Push}'' is a robotic pushing environment adopted from~\cite{nair2018visual} simulated with MuJoCo~\cite{todorov2012mujoco} engine. Each episode lasts for at most 50 steps.

\textbf{Observation: } The agent can observe 2-D hand position and 2-D puck position in each step. 

\textbf{Action space: } An action is a 2-D vector denoting the relative movement of the robot hand. The height of the hand and the state of the gripper are fixed. Each dimension ranges from -1 to 1, and is categorized into 3 bins (therefore, there are a total of 9 different actions).

\textbf{Goal and reward: } The puck goal is a 2-D vector denoting the target position on the table. The agent can only get reward when the $l2$ distance between puck and goal is smaller than 5cm.

\textbf{Initial state: } Each episode starts with the hand initialized in a 0.03m $\times$ 0.03m region on the right side of the table, the puck initialized in a 0.07m $\times$ 0.07m square, and the goal is uniformly sampled in a 0.4m $\times$ 0.4m square on the table. 

\subsubsection{Ant Maze}
This environment is adopted from~\cite{nachum2018data}. The scene is a 24m $\times$ 24m maze simulated with MuJoCo~\cite{todorov2012mujoco}. The maximum number of steps for each episode is 500 steps.

\textbf{Observation:} Joint positions and joint velocities of the ant robot. The center of mass of the robot is included in the joint positions.

\textbf{Action:} 8-D real-valued vector controlling the motors on the robot. The actions output from the policy network are clipped to -30 to 30 before sent to the simulator.

\textbf{Goal and reward:} The goals are 2-D vectors denoting the desired $xy$ position of the ant. The agent gets reward when the distance between the ant's center of mass and the goal is smaller than 2m.

\textbf{Initial state:} Both the ant and the goal are uniformly sampled in the empty space (not collided with walls) in the maze. For hard tasks, the ant is initialized at coordinate (0m, 0m), and the goal is at coordinate (0m, 16m).

\subsubsection{Stacking}\label{app:stack_train_detail}
The stacking environment is built with PyBullet~\cite{coumans2016pybullet}. 

\textbf{Observation:} The observation is a concatenation of robot state and objects states. The robot state contains 6-D end effector pose, 3-D end effector linear velocity, 1-D finger position, and 1-D finger velocity. Each object state consists of its 6-D pose, 3-D relative position to the robot's end effector, 3-D linear velocity, 3-D angular velocity and 1-D 0/1 indicator denoting the identity of the target cube.

\textbf{Action: } A concatenation of 3-D relative movement of the end effector and 1-D finger control signal. Each dimension is discretized into 20 bins.

\textbf{Goal and reward: } A goal specifies a 3-D desired position and an identity denoting which cube must be close to the goal spot. A non-zero reward is only given when the target cube is within 3cm distance to the goal, and when the end effector is at least 10cm away from the goal.

\textbf{Adaptive curriculum:}
% \yf{curriculum, task reduction, architecture}
Since directly training on stacking tasks with large number of objects pose severe exploration challenges, we adopt a simple curriculum on the number of cubes to stack, similar to the one in~\cite{li2020towards}. In the beginning, 70\% of the tasks only need to stack one cube, and the other tasks are uniformly sampled to stack 2 to $N$ cubes. 
Whenever the average success rate reaches 0.6, the curriculum proceeds from sampling ``stack-$n$-cubes'' with 70\% probability to ``stack-($n+1$)-cubes'' tasks. 
We switch to offline SL training when the curriculum proceeds and the success rate for stacking $n+1$ cubes is lower than 0.5. 

\subsection{Combination with Offline RL}
\input{82expr_offline}

\subsection{Algorithm Implementations}\label{app:pair_implement}
\textbf{Network architecture:}
We use separate policy and value networks with the same architecture for PPO-based algorithms. The specific network architectures for different domains are as follows. For ``Push'', we use MLPs of hidden size 400 and 300. For ``Ant Maze'', the networks are MLPs of hidden size 256 and 256. For ``Stacking'', we use a transformer~\cite{vaswani2017attention}-based architecture, which stacks 2 self-attention blocks with one head and 64 hidden units, and then goes though linear heads to output action distributions or values. Since there are exponential number of possible actions as we discretize each action dimension into several bins, we assume different action dimensions are independent, and use separate linear heads to predict distributions of different dimensions.

\textbf{{\name} implementation:}
We use PPO~\cite{schulman17ppo} for online training and advantage weighted imitation learning for offline training by default. For PPO training, we use $N_{\textrm{worker}}$ parallel workers to collect transitions $(s, a, r, s')$ from the environments synchronously. When applying the value-difference-based intrinsic reward, $r$ is replaced by the sum of environment reward and the intrinsic reward calculated with the current value network. After each worker gets $N_{\textrm{steps}}$ data points, we run $N_{\textrm{epoch}}$ epochs of PPO training by splitting all the collected on-policy data into $N_{\textrm{batch}}$ batches. The successful trajectories from these on-policy batch are directly stored into the offline data $\mathcal{D}$, and the failed trajectories are cached into a failure buffer $\mathcal{B}_{\textrm{fail}}$. We repeat the data collection - PPO training phase until the total amount of on-policy data reaches $N_{\textrm{online}}$. Then we perform data augmentation (task reduction, goal relabeling) over $\mathcal{B}_{\textrm{fail}}$ so as to generate more successful data, and insert the resulting demonstrations into $\mathcal{D}$ as well. After the dataset $\mathcal{D}$ is constructed, we run $M_{\textrm{epoch}}$ epochs of advantage weighted 
behavior cloning with batch size $m$. We use GAE advantage calculated with the value network learned after the online phase as $(R - V)$ for each data point. We use Adam optimizer~\cite{kingma2015adam} for PPO and supervised learning. All the hyper-parameters are listed in Table~\ref{tab:app:pair_hyper}.  

We do not additionally train the value network using the offline dataset in offline phase since we empirically find that training $V$ does not help the overall performance (see Figure~\ref{fig:app:vtrain}). 
\begin{figure}
    \centering
    \includegraphics[width=0.3\textwidth]{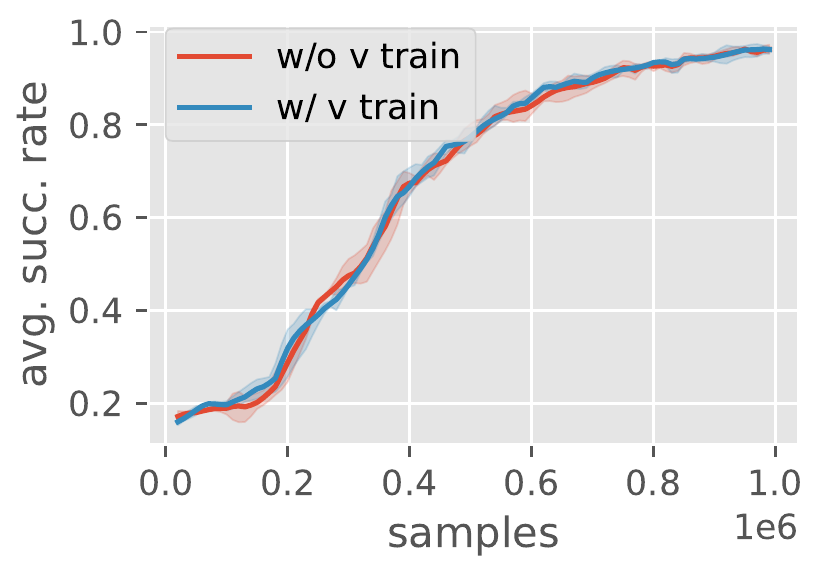}
    \includegraphics[width=0.3\textwidth]{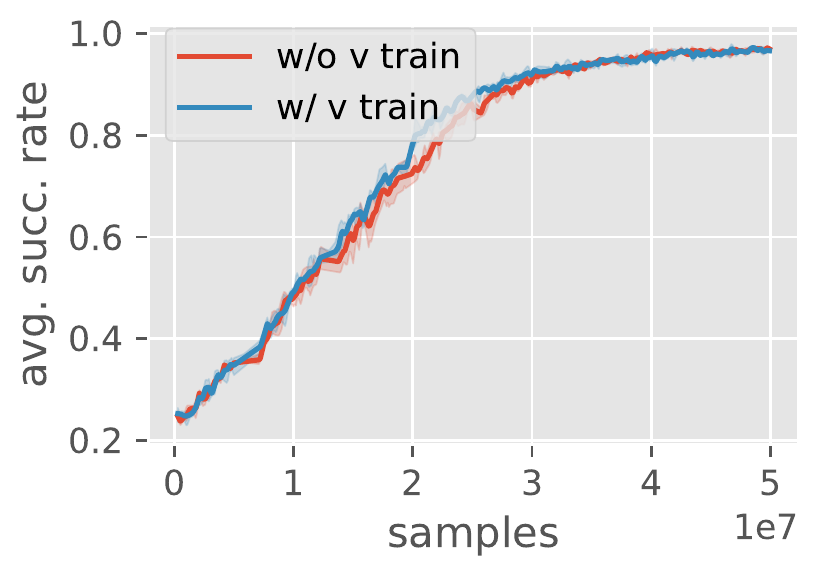}
    \caption{Comparison of training value network or not in {\name} offline phase. Left plot is in ``Push'' domain, right plot is in ``Ant Maze''.}
    \label{fig:app:vtrain}
    \vspace{-2mm}
\end{figure}

\begin{table}[]
    \centering
    \begin{tabular}{lccc}
    \toprule
        Domain & Push & Ant Maze & Stacking \\
    \midrule
        $N_{\textrm{worker}}$ & 4 & 64 & 64 \\
        $N_{\textrm{online}}$ & $10 \cdot 2^{14}$ & $20\cdot 2^{18}$ & adaptive \\
        \hline
        $\alpha$ &
        \multicolumn{3}{c}{0.5} \\
        % \hline
        $N_{\textrm{steps}}$ & \multicolumn{3}{c}{4096} \\
        % \hline
        $N_{\textrm{epoch}}$ & \multicolumn{3}{c}{10}\\
        % \hline
        $N_{\textrm{batch}}$ & \multicolumn{3}{c}{32} \\
        $\beta$ & \multicolumn{3}{c}{1} \\
        $M_{\textrm{epoch}}$ & \multicolumn{3}{c}{10} \\
        $m$ & \multicolumn{3}{c}{64} \\
        lr & \multicolumn{3}{c}{2.5e-4} \\
    \bottomrule  
    \end{tabular}
    \caption{Hyperparameters of {\name}.}
    \label{tab:app:pair_hyper}
    \vspace{-2mm}
\end{table}

\textbf{GCSL baseline implementation:}
There is a slight difference between our implementation and the original version from \cite{ghosh2020learning}: we collect and relabel data in a more phasic fashion, i.e., we perform imitation learning on the data batch collected from only the previous online collection phase. We keep the number of online steps before offline imitation learning the same as {\name}. The original GCSL maintains a data buffer throughout the training process similar to the setting in off-policy RL. We find our phasic GCSL can get better performance than the original version (see Figure~\ref{fig:app:gcsl}), so we present the results of phasic GCSL in the main paper.

% \begin{figure}
%     \centering
%     \includegraphics[width=0.3\textwidth]{fig/sawyerpush/SawyerPush-v1_gcsl.pdf}
%     \caption{Comparison between phasic GCSL and the original version.}
%     \label{fig:app:gcsl}
% \end{figure}
\begin{figure}
    \begin{minipage}{0.49\textwidth}
    \centering
    \includegraphics[width=0.6\linewidth]{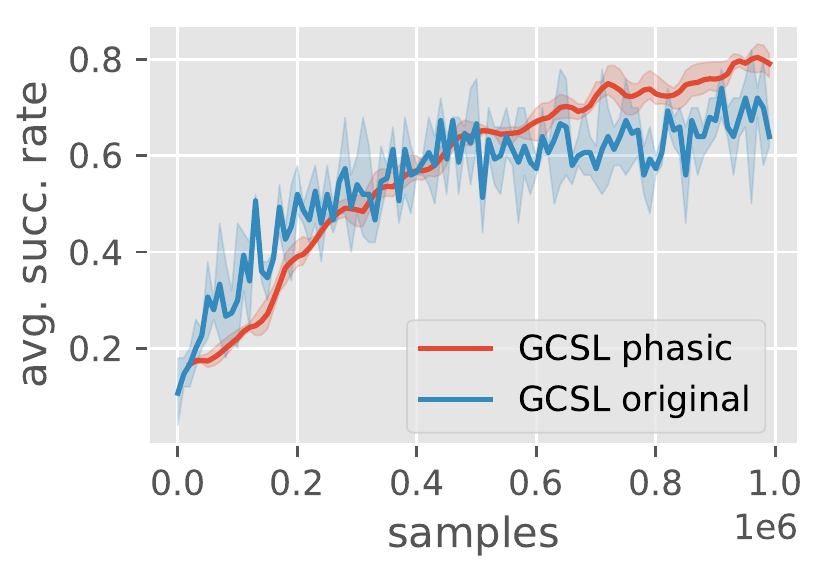}
    \caption{Comparison between phasic GCSL and the original version.}
    \label{fig:app:gcsl}
    \end{minipage}
    \hspace{2mm}
    \begin{minipage}{0.49\textwidth}
    \centering
    \includegraphics[width=0.6\linewidth]{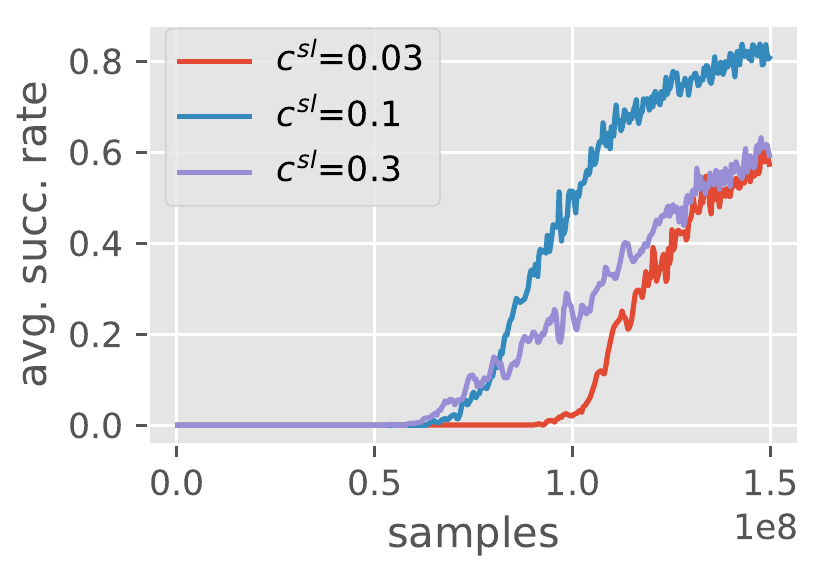}
    \caption{Success rate of joint RL+SL with different $c^{SL}$ in ``stack-6-cube'' task. We adopt $c^{SL}$ = 0.1 for SIR in the main paper.}
    \label{fig:app:csl}
    \end{minipage}
    \vspace{-3mm}
\end{figure}
\textbf{DT baseline implementation:}
The policy and value networks adopt the same GPT-2 architecture similar to the original version in \cite{chen2021decision}. They output actions and values conditioning on desired return, past states, and past actions. To avoid gradient explosion, only the last token of the predicted action or value sequence is used for updating the model. To stabilize training, we remove all dropout layers from the transformer model. We use a context length of 5 for sequence conditioning. The specific feature extractors for processing single-step observations in different domains are as follows. For ``Push'', we use MLPs of hidden size 300 and 300 for representation learning and we train separate policy and value networks. For ``Stacking'', we use a transformer-based architecture similar to {\name} for representation learning except that the size of hidden layer is 128. We find that using shared or separate feature extractors for policy and value networks leads to similar performances. Therefore we share parameters for them to save computational resources.

\subsection{Analysis on Phasic vs. Joint RL and SL Optimization}\label{sec:app:phasic}
{\name} decouples RL and SL objectives in two phases instead of optimizing them jointly ($L^{RL} + c^{SL}L^{SL}$), since the two objectives can largely interfere with each other and the choice of $c^{SL}$ is empirically brittle to tune. The RL objective is policy gradient over rollout data (Eqn.~\ref{eq:rl}), which requires (primarily) on-policy samples (both success and failures) to make policy improvement. The SL objective (Eqn.~\ref{eq:bc}) is performed over the successful dataset with both success-only rollout samples and off-policy augmentation trajectories. These two objectives operate on very distinct data distributions. If $c^{SL}$ is too large, the gradient will be pulled away from the policy improvement direction, which makes policy learning unstable or even breaks training. If $c^{SL}$ is too small, the objective may not sufficiently leverage the augmented successful data learning to slow convergence. We report sensitivity analysis in Fig.~\ref{fig:app:csl} by trying different $c^{SL}$ in the joint objective. 
With phasic training, RL and SL are decoupled and the interference is largely reduced. 
\yf{We can also intuitively interpret the benefit of decoupling RL and SL into two phases by echoing one result in stochastic optimization: smaller optimization step size should be taken when the gradient noise is large (Theorem 6.2 in ~\cite{bubeck15convex}). RL objective is with large gradient noise, while SL offers clear supervision signal. When jointly optimizing RL and SL, only small step size is allowed which leads to slow convergence; when optimizing RL and SL separately, different step sizes can be chosen for different objectives, thus converges faster.}
% $\mathbb{E}f(\frac{1}{t}\sigma_{s=1}^t x_{s+1})-f(x^\star) \leq \eta \sigma^2 + \frac{\beta R^2}{t}.$

%% file: 82expr_offline.tex
% \yf{unexplained things like dataset sources, sizes; DT results on other domains}
\subsubsection{Offline RL combination in Sawyer Push}
For the warm-start dataset, we collect rollout trajectories generated by random policy, where we keep successful samples and do goal relabelling on failed trajectories. Therefore, the dataset only consists of successful demonstrations. The dataset size for phasic-DT is $55K$ steps, and $15k$ for phasic-AWAC. Phasic-AWAC switches from online to offline training after collecting $200k$ online samples to match the total number of phasic switches in original {\name}. 
The online training of Phasic-DT is based on PPO, therefore we adopt the same number of PPO updates as {\name} for phasic-DT before switching to its offline phase. 
% Phasic-DT switches from online to offline phase every $80k$ environment steps since we find phasic-DT needs more frequent offline updating to accelerate increasing rate of performance empirically.

\subsubsection{Additional results in Stack scenario}
We also try to run phasic-DT in the stacking domain. We train phasic-DT starting from the offline phase with a warm-up dataset. The warm-up dataset consists of successful trajectories in ``stack-1-cube" scenario generated by a pretrained PPO model. We similarly adopt adaptive curriculum to alleviate exploration difficulty as when training {\name}. The results are shown in Figure~\ref{fig:app:dt_stack}. The success rate of phasic-DT for stacking 6 cubes grows slowly, and does not converge to a value as high as the original {\name}. 
\begin{figure}
    \centering
    \includegraphics[width=0.3\textwidth]{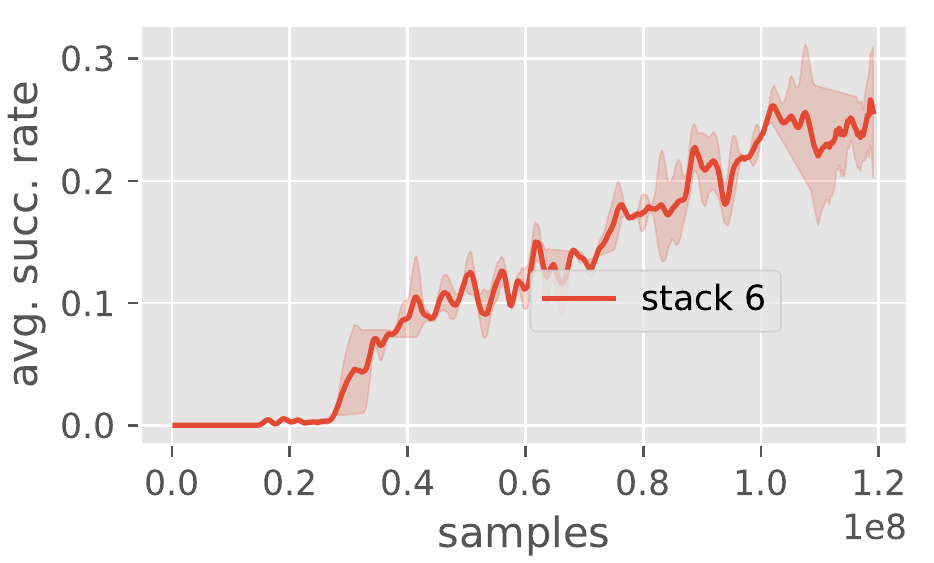}
    \caption{Average success rate of phasic-DT in ``stack-6-cube'' domain.}
    \label{fig:app:dt_stack}
    \vspace{-2mm}
\end{figure}

%% file: paper.bbl
\begin{thebibliography}{69}
\providecommand{\natexlab}[1]{#1}
\providecommand{\url}[1]{\texttt{#1}}
\expandafter\ifx\csname urlstyle\endcsname\relax
  \providecommand{\doi}[1]{doi: #1}\else
  \providecommand{\doi}{doi: \begingroup \urlstyle{rm}\Url}\fi

\bibitem[Akkaya et~al.(2019)Akkaya, Andrychowicz, Chociej, Litwin, McGrew,
  Petron, Paino, Plappert, Powell, Ribas, et~al.]{akkaya2019solving}
Akkaya, I., Andrychowicz, M., Chociej, M., Litwin, M., McGrew, B., Petron, A.,
  Paino, A., Plappert, M., Powell, G., Ribas, R., et~al.
\newblock Solving rubik's cube with a robot hand.
\newblock \emph{arXiv preprint arXiv:1910.07113}, 2019.

\bibitem[Andrychowicz et~al.(2017)Andrychowicz, Wolski, Ray, Schneider, Fong,
  Welinder, McGrew, Tobin, Abbeel, and Zaremba]{andrychowicz2017hindsight}
Andrychowicz, M., Wolski, F., Ray, A., Schneider, J., Fong, R., Welinder, P.,
  McGrew, B., Tobin, J., Abbeel, P., and Zaremba, W.
\newblock Hindsight experience replay.
\newblock In \emph{Proceedings of the 31st International Conference on Neural
  Information Processing Systems}, pp.\  5055--5065, 2017.

\bibitem[Andrychowicz et~al.(2021)Andrychowicz, Raichuk, Sta{\'n}czyk, Orsini,
  Girgin, Marinier, Hussenot, Geist, Pietquin, Michalski, Gelly, and
  Bachem]{andrychowicz2021what}
Andrychowicz, M., Raichuk, A., Sta{\'n}czyk, P., Orsini, M., Girgin, S.,
  Marinier, R., Hussenot, L., Geist, M., Pietquin, O., Michalski, M., Gelly,
  S., and Bachem, O.
\newblock What matters for on-policy deep actor-critic methods? a large-scale
  study.
\newblock In \emph{International Conference on Learning Representations}, 2021.

\bibitem[Azuma(1967)]{azuma1967weighted}
Azuma, K.
\newblock Weighted sums of certain dependent random variables.
\newblock \emph{Tohoku Mathematical Journal, Second Series}, 19\penalty0
  (3):\penalty0 357--367, 1967.

\bibitem[Bacon et~al.(2017)Bacon, Harb, and Precup]{bacon2017option}
Bacon, P.-L., Harb, J., and Precup, D.
\newblock The option-critic architecture.
\newblock In \emph{Proceedings of the Thirty-First AAAI Conference on
  Artificial Intelligence}, pp.\  1726--1734, 2017.

\bibitem[Badia et~al.(2020)Badia, Sprechmann, Vitvitskyi, Guo, Piot,
  Kapturowski, Tieleman, Arjovsky, Pritzel, Bolt, and Blundell]{Badia2020Never}
Badia, A.~P., Sprechmann, P., Vitvitskyi, A., Guo, D., Piot, B., Kapturowski,
  S., Tieleman, O., Arjovsky, M., Pritzel, A., Bolt, A., and Blundell, C.
\newblock Never give up: Learning directed exploration strategies.
\newblock In \emph{International Conference on Learning Representations}, 2020.
\newblock URL \url{https://openreview.net/forum?id=Sye57xStvB}.

\bibitem[Brown et~al.(2020)Brown, Mann, Ryder, Subbiah, Kaplan, Dhariwal,
  Neelakantan, Shyam, Sastry, Askell, Agarwal, Herbert-Voss, Krueger, Henighan,
  Child, Ramesh, Ziegler, Wu, Winter, Hesse, Chen, Sigler, Litwin, Gray, Chess,
  Clark, Berner, McCandlish, Radford, Sutskever, and
  Amodei]{larochelle20language}
Brown, T., Mann, B., Ryder, N., Subbiah, M., Kaplan, J.~D., Dhariwal, P.,
  Neelakantan, A., Shyam, P., Sastry, G., Askell, A., Agarwal, S.,
  Herbert-Voss, A., Krueger, G., Henighan, T., Child, R., Ramesh, A., Ziegler,
  D., Wu, J., Winter, C., Hesse, C., Chen, M., Sigler, E., Litwin, M., Gray,
  S., Chess, B., Clark, J., Berner, C., McCandlish, S., Radford, A., Sutskever,
  I., and Amodei, D.
\newblock Language models are few-shot learners.
\newblock In Larochelle, H., Ranzato, M., Hadsell, R., Balcan, M.~F., and Lin,
  H. (eds.), \emph{Advances in Neural Information Processing Systems},
  volume~33, pp.\  1877--1901. Curran Associates, Inc., 2020.

\bibitem[Bubeck(2015)]{bubeck15convex}
Bubeck, S.
\newblock Convex optimization: Algorithms and complexity.
\newblock \emph{Found. Trends Mach. Learn.}, 8\penalty0 (3-4):\penalty0
  231--357, 2015.
\newblock \doi{10.1561/2200000050}.

\bibitem[Burda et~al.(2019)Burda, Edwards, Storkey, and
  Klimov]{burda2018exploration}
Burda, Y., Edwards, H., Storkey, A., and Klimov, O.
\newblock Exploration by random network distillation.
\newblock In \emph{International Conference on Learning Representations}, 2019.
\newblock URL \url{https://openreview.net/forum?id=H1lJJnR5Ym}.

\bibitem[Chen et~al.(2021)Chen, Lu, Rajeswaran, Lee, Grover, Laskin, Abbeel,
  Srinivas, and Mordatch]{chen2021decision}
Chen, L., Lu, K., Rajeswaran, A., Lee, K., Grover, A., Laskin, M., Abbeel, P.,
  Srinivas, A., and Mordatch, I.
\newblock Decision transformer: Reinforcement learning via sequence modeling.
\newblock \emph{arXiv preprint arXiv:2106.01345}, 2021.

\bibitem[Cobbe et~al.(2021)Cobbe, Hilton, Klimov, and Schulman]{cobbe21ppg}
Cobbe, K., Hilton, J., Klimov, O., and Schulman, J.
\newblock Phasic policy gradient.
\newblock In Meila, M. and Zhang, T. (eds.), \emph{Proceedings of the 38th
  International Conference on Machine Learning, {ICML} 2021, 18-24 July 2021,
  Virtual Event}, volume 139 of \emph{Proceedings of Machine Learning
  Research}, pp.\  2020--2027. {PMLR}, 2021.

\bibitem[Coumans \& Bai(2016)Coumans and Bai]{coumans2016pybullet}
Coumans, E. and Bai, Y.
\newblock Pybullet, a python module for physics simulation for games, robotics
  and machine learning.
\newblock 2016.

\bibitem[Dosovitskiy et~al.(2020)Dosovitskiy, Beyer, Kolesnikov, Weissenborn,
  Zhai, Unterthiner, Dehghani, Minderer, Heigold, Gelly,
  et~al.]{dosovitskiy2020image}
Dosovitskiy, A., Beyer, L., Kolesnikov, A., Weissenborn, D., Zhai, X.,
  Unterthiner, T., Dehghani, M., Minderer, M., Heigold, G., Gelly, S., et~al.
\newblock An image is worth 16x16 words: Transformers for image recognition at
  scale.
\newblock In \emph{International Conference on Learning Representations}, 2020.

\bibitem[Ecoffet et~al.(2021)Ecoffet, Huizinga, Lehman, Stanley, and
  Clune]{ecoffet2021first}
Ecoffet, A., Huizinga, J., Lehman, J., Stanley, K.~O., and Clune, J.
\newblock First return, then explore.
\newblock \emph{Nature}, 590\penalty0 (7847):\penalty0 580--586, 2021.

\bibitem[Emmons et~al.(2021)Emmons, Eysenbach, Kostrikov, and
  Levine]{emmons2021rvs}
Emmons, S., Eysenbach, B., Kostrikov, I., and Levine, S.
\newblock Rvs: What is essential for offline rl via supervised learning?
\newblock \emph{arXiv preprint arXiv:2112.10751}, 2021.

\bibitem[Engstrom et~al.(2020)Engstrom, Ilyas, Santurkar, Tsipras, Janoos,
  Rudolph, and Madry]{Engstrom2020Implementation}
Engstrom, L., Ilyas, A., Santurkar, S., Tsipras, D., Janoos, F., Rudolph, L.,
  and Madry, A.
\newblock Implementation matters in deep rl: A case study on ppo and trpo.
\newblock In \emph{International Conference on Learning Representations}, 2020.

\bibitem[Eysenbach et~al.(2020)Eysenbach, Salakhutdinov, and
  Levine]{eysenbach2020c}
Eysenbach, B., Salakhutdinov, R., and Levine, S.
\newblock C-learning: Learning to achieve goals via recursive classification.
\newblock In \emph{International Conference on Learning Representations}, 2020.

\bibitem[Fujimoto \& Gu(2021)Fujimoto and Gu]{fujimoto2021minimalist}
Fujimoto, S. and Gu, S.~S.
\newblock A minimalist approach to offline reinforcement learning.
\newblock \emph{arXiv preprint arXiv:2106.06860}, 2021.

\bibitem[Furuta et~al.(2021)Furuta, Matsuo, and Gu]{furuta2021generalized}
Furuta, H., Matsuo, Y., and Gu, S.~S.
\newblock Generalized decision transformer for offline hindsight information
  matching.
\newblock In \emph{Deep RL Workshop NeurIPS 2021}, 2021.

\bibitem[Ghosh et~al.(2020)Ghosh, Gupta, Reddy, Fu, Devin, Eysenbach, and
  Levine]{ghosh2020learning}
Ghosh, D., Gupta, A., Reddy, A., Fu, J., Devin, C.~M., Eysenbach, B., and
  Levine, S.
\newblock Learning to reach goals via iterated supervised learning.
\newblock In \emph{International Conference on Learning Representations}, 2020.

\bibitem[Hester et~al.(2018)Hester, Vecerik, Pietquin, Lanctot, Schaul, Piot,
  Horgan, Quan, Sendonaris, Osband, et~al.]{hester2018deep}
Hester, T., Vecerik, M., Pietquin, O., Lanctot, M., Schaul, T., Piot, B.,
  Horgan, D., Quan, J., Sendonaris, A., Osband, I., et~al.
\newblock Deep q-learning from demonstrations.
\newblock In \emph{Thirty-second AAAI conference on artificial intelligence},
  2018.

\bibitem[Hwangbo et~al.(2019)Hwangbo, Lee, Dosovitskiy, Bellicoso, Tsounis,
  Koltun, and Hutter]{hwangbo2019learning}
Hwangbo, J., Lee, J., Dosovitskiy, A., Bellicoso, D., Tsounis, V., Koltun, V.,
  and Hutter, M.
\newblock Learning agile and dynamic motor skills for legged robots.
\newblock \emph{Science Robotics}, 4\penalty0 (26), 2019.

\bibitem[Ilyas et~al.(2020)Ilyas, Engstrom, Santurkar, Tsipras, Janoos,
  Rudolph, and Madry]{Ilyas2020A}
Ilyas, A., Engstrom, L., Santurkar, S., Tsipras, D., Janoos, F., Rudolph, L.,
  and Madry, A.
\newblock A closer look at deep policy gradients.
\newblock In \emph{International Conference on Learning Representations}, 2020.

\bibitem[Janner et~al.(2021)Janner, Li, and Levine]{janner2021offline}
Janner, M., Li, Q., and Levine, S.
\newblock Offline reinforcement learning as one big sequence modeling problem.
\newblock \emph{Advances in Neural Information Processing Systems}, 34, 2021.

\bibitem[Jeon \& Kim(2020)Jeon and Kim]{jeon2020autonomous}
Jeon, W. and Kim, D.
\newblock Autonomous molecule generation using reinforcement learning and
  docking to develop potential novel inhibitors.
\newblock \emph{Scientific reports}, 10\penalty0 (1):\penalty0 1--11, 2020.

\bibitem[Jumper et~al.(2021)Jumper, Evans, Pritzel, Green, Figurnov,
  Ronneberger, Tunyasuvunakool, Bates, {\v{Z}}{\'\i}dek, Potapenko,
  et~al.]{jumper2021highly}
Jumper, J., Evans, R., Pritzel, A., Green, T., Figurnov, M., Ronneberger, O.,
  Tunyasuvunakool, K., Bates, R., {\v{Z}}{\'\i}dek, A., Potapenko, A., et~al.
\newblock Highly accurate protein structure prediction with alphafold.
\newblock \emph{Nature}, 596\penalty0 (7873):\penalty0 583--589, 2021.

\bibitem[Kaelbling(1993)]{kaelbling1993learning}
Kaelbling, L.
\newblock Learning to achieve goals.
\newblock In \emph{Proc. of IJCAI-93}, pp.\  1094--1098, 1993.

\bibitem[Kidambi et~al.(2020)Kidambi, Rajeswaran, Netrapalli, and
  Joachims]{kidambi2020morel}
Kidambi, R., Rajeswaran, A., Netrapalli, P., and Joachims, T.
\newblock Morel: Model-based offline reinforcement learning.
\newblock In \emph{NeurIPS}, 2020.

\bibitem[Kingma \& Ba(2015)Kingma and Ba]{kingma2015adam}
Kingma, D.~P. and Ba, J.
\newblock Adam: A method for stochastic optimization.
\newblock In \emph{ICLR (Poster)}, 2015.

\bibitem[Kostrikov et~al.(2021)Kostrikov, Nair, and Levine]{kostrikov2021iql}
Kostrikov, I., Nair, A., and Levine, S.
\newblock Offline reinforcement learning with implicit q-learning.
\newblock \emph{arXiv preprint arXiv:2110.06169}, 2021.

\bibitem[Kulkarni et~al.(2016)Kulkarni, Narasimhan, Saeedi, and
  Tenenbaum]{kulkarni2016hierarchical}
Kulkarni, T.~D., Narasimhan, K., Saeedi, A., and Tenenbaum, J.
\newblock Hierarchical deep reinforcement learning: Integrating temporal
  abstraction and intrinsic motivation.
\newblock \emph{Advances in neural information processing systems},
  29:\penalty0 3675--3683, 2016.

\bibitem[Kumar et~al.(2019)Kumar, Peng, and Levine]{kumar2019reward}
Kumar, A., Peng, X.~B., and Levine, S.
\newblock Reward-conditioned policies.
\newblock \emph{arXiv preprint arXiv:1912.13465}, 2019.

\bibitem[Kumar et~al.(2020)Kumar, Zhou, Tucker, and Levine]{kumar20cql}
Kumar, A., Zhou, A., Tucker, G., and Levine, S.
\newblock Conservative q-learning for offline reinforcement learning.
\newblock In Larochelle, H., Ranzato, M., Hadsell, R., Balcan, M., and Lin, H.
  (eds.), \emph{Advances in Neural Information Processing Systems 33: Annual
  Conference on Neural Information Processing Systems 2020, NeurIPS 2020,
  December 6-12, 2020, virtual}, 2020.

\bibitem[Lange et~al.(2012)Lange, Gabel, and Riedmiller]{lange2012batch}
Lange, S., Gabel, T., and Riedmiller, M.
\newblock Batch reinforcement learning.
\newblock In \emph{Reinforcement learning}, pp.\  45--73. Springer, 2012.

\bibitem[Levine(2021)]{levine21understanding}
Levine, S.
\newblock Understanding the world through action.
\newblock In Faust, A., Hsu, D., and Neumann, G. (eds.), \emph{Conference on
  Robot Learning, 8-11 November 2021, London, {UK}}, volume 164 of
  \emph{Proceedings of Machine Learning Research}, pp.\  1752--1757. {PMLR},
  2021.

\bibitem[Levine et~al.(2020)Levine, Kumar, Tucker, and Fu]{levine2020offline}
Levine, S., Kumar, A., Tucker, G., and Fu, J.
\newblock Offline reinforcement learning: Tutorial, review, and perspectives on
  open problems.
\newblock \emph{arXiv preprint arXiv:2005.01643}, 2020.

\bibitem[Li et~al.(2020{\natexlab{a}})Li, Jabri, Darrell, and
  Agrawal]{li2020towards}
Li, R., Jabri, A., Darrell, T., and Agrawal, P.
\newblock Towards practical multi-object manipulation using relational
  reinforcement learning.
\newblock In \emph{2020 IEEE International Conference on Robotics and
  Automation (ICRA)}, pp.\  4051--4058. IEEE, 2020{\natexlab{a}}.

\bibitem[Li et~al.(2020{\natexlab{b}})Li, Wu, Xu, Wang, and Wu]{li2020solving}
Li, Y., Wu, Y., Xu, H., Wang, X., and Wu, Y.
\newblock Solving compositional reinforcement learning problems via task
  reduction.
\newblock In \emph{International Conference on Learning Representations},
  2020{\natexlab{b}}.

\bibitem[Lillicrap et~al.(2016)Lillicrap, Hunt, Pritzel, Heess, Erez, Tassa,
  Silver, and Wierstra]{lillicrap2016continuous}
Lillicrap, T.~P., Hunt, J.~J., Pritzel, A., Heess, N., Erez, T., Tassa, Y.,
  Silver, D., and Wierstra, D.
\newblock Continuous control with deep reinforcement learning.
\newblock In \emph{ICLR (Poster)}, 2016.

\bibitem[Lu et~al.(2021)Lu, Hausman, Chebotar, Yan, Jang, Herzog, Xiao, Irpan,
  Khansari, Kalashnikov, and Levine]{lu2021awopt}
Lu, Y., Hausman, K., Chebotar, Y., Yan, M., Jang, E., Herzog, A., Xiao, T.,
  Irpan, A., Khansari, M., Kalashnikov, D., and Levine, S.
\newblock {AW}-opt: Learning robotic skills with imitation andreinforcement at
  scale.
\newblock In \emph{5th Annual Conference on Robot Learning}, 2021.
\newblock URL \url{https://openreview.net/forum?id=xwEaXgFa0MR}.

\bibitem[Lynch et~al.(2020)Lynch, Khansari, Xiao, Kumar, Tompson, Levine, and
  Sermanet]{lynch2020learning}
Lynch, C., Khansari, M., Xiao, T., Kumar, V., Tompson, J., Levine, S., and
  Sermanet, P.
\newblock Learning latent plans from play.
\newblock In \emph{Conference on Robot Learning}, pp.\  1113--1132. PMLR, 2020.

\bibitem[Mao et~al.(2022)Mao, Wang, Wang, and Zhang]{mao2022moore}
Mao, Y., Wang, C., Wang, B., and Zhang, C.
\newblock Moore: Model-based offline-to-online reinforcement learning, 2022.

\bibitem[Matsushima et~al.(2020)Matsushima, Furuta, Matsuo, Nachum, and
  Gu]{matsushima2020deployment}
Matsushima, T., Furuta, H., Matsuo, Y., Nachum, O., and Gu, S.
\newblock Deployment-efficient reinforcement learning via model-based offline
  optimization.
\newblock In \emph{International Conference on Learning Representations}, 2020.

\bibitem[Mnih et~al.(2015)Mnih, Kavukcuoglu, Silver, Rusu, Veness, Bellemare,
  Graves, Riedmiller, Fidjeland, Ostrovski, et~al.]{mnih2015human}
Mnih, V., Kavukcuoglu, K., Silver, D., Rusu, A.~A., Veness, J., Bellemare,
  M.~G., Graves, A., Riedmiller, M., Fidjeland, A.~K., Ostrovski, G., et~al.
\newblock Human-level control through deep reinforcement learning.
\newblock \emph{nature}, 518\penalty0 (7540):\penalty0 529--533, 2015.

\bibitem[Nachum et~al.(2018)Nachum, Gu, Lee, and Levine]{nachum2018data}
Nachum, O., Gu, S.~S., Lee, H., and Levine, S.
\newblock Data-efficient hierarchical reinforcement learning.
\newblock \emph{Advances in Neural Information Processing Systems},
  31:\penalty0 3303--3313, 2018.

\bibitem[Nair et~al.(2018{\natexlab{a}})Nair, McGrew, Andrychowicz, Zaremba,
  and Abbeel]{nair2018overcoming}
Nair, A., McGrew, B., Andrychowicz, M., Zaremba, W., and Abbeel, P.
\newblock Overcoming exploration in reinforcement learning with demonstrations.
\newblock In \emph{2018 IEEE International Conference on Robotics and
  Automation (ICRA)}, pp.\  6292--6299. IEEE, 2018{\natexlab{a}}.

\bibitem[Nair et~al.(2020)Nair, Dalal, Gupta, and Levine]{nair2020awac}
Nair, A., Dalal, M., Gupta, A., and Levine, S.
\newblock Awac: Accelerating online reinforcement learning with offline
  datasets.
\newblock 2020.

\bibitem[Nair et~al.(2018{\natexlab{b}})Nair, Pong, Dalal, Bahl, Lin, and
  Levine]{nair2018visual}
Nair, A.~V., Pong, V., Dalal, M., Bahl, S., Lin, S., and Levine, S.
\newblock Visual reinforcement learning with imagined goals.
\newblock \emph{Advances in Neural Information Processing Systems},
  31:\penalty0 9191--9200, 2018{\natexlab{b}}.

\bibitem[Nasiriany et~al.(2019)Nasiriany, Pong, Lin, and
  Levine]{nasiriany2019planning}
Nasiriany, S., Pong, V., Lin, S., and Levine, S.
\newblock Planning with goal-conditioned policies.
\newblock \emph{Advances in Neural Information Processing Systems},
  32:\penalty0 14843--14854, 2019.

\bibitem[Oh et~al.(2018)Oh, Guo, Singh, and Lee]{oh2018self}
Oh, J., Guo, Y., Singh, S., and Lee, H.
\newblock Self-imitation learning.
\newblock In \emph{International Conference on Machine Learning}, pp.\
  3878--3887. PMLR, 2018.

\bibitem[Peng et~al.(2019)Peng, Kumar, Zhang, and Levine]{peng2019advantage}
Peng, X.~B., Kumar, A., Zhang, G., and Levine, S.
\newblock Advantage-weighted regression: Simple and scalable off-policy
  reinforcement learning.
\newblock \emph{arXiv preprint arXiv:1910.00177}, 2019.

\bibitem[Pong et~al.(2018)Pong, Gu, Dalal, and Levine]{pong2018temporal}
Pong, V., Gu, S., Dalal, M., and Levine, S.
\newblock Temporal difference models: Model-free deep rl for model-based
  control.
\newblock In \emph{International Conference on Learning Representations}, 2018.

\bibitem[Rajeswaran et~al.(2018)Rajeswaran, Kumar, Gupta, Vezzani, Schulman,
  Todorov, and Levine]{rajeswaran18dexterous}
Rajeswaran, A., Kumar, V., Gupta, A., Vezzani, G., Schulman, J., Todorov, E.,
  and Levine, S.
\newblock Learning complex dexterous manipulation with deep reinforcement
  learning and demonstrations.
\newblock In Kress{-}Gazit, H., Srinivasa, S.~S., Howard, T., and Atanasov, N.
  (eds.), \emph{Robotics: Science and Systems XIV, Carnegie Mellon University,
  Pittsburgh, Pennsylvania, USA, June 26-30, 2018}, 2018.
\newblock \doi{10.15607/RSS.2018.XIV.049}.
\newblock URL \url{http://www.roboticsproceedings.org/rss14/p49.html}.

\bibitem[Rashidinejad et~al.(2021)Rashidinejad, Zhu, Ma, Jiao, and
  Russell]{rashidinejad2021bridging}
Rashidinejad, P., Zhu, B., Ma, C., Jiao, J., and Russell, S.
\newblock Bridging offline reinforcement learning and imitation learning: A
  tale of pessimism.
\newblock \emph{arXiv preprint arXiv:2103.12021}, 2021.

\bibitem[Schaul et~al.(2015)Schaul, Horgan, Gregor, and
  Silver]{schaul2015universal}
Schaul, T., Horgan, D., Gregor, K., and Silver, D.
\newblock Universal value function approximators.
\newblock In \emph{International conference on machine learning}, pp.\
  1312--1320. PMLR, 2015.

\bibitem[Schmidhuber(2019)]{schmidhuber2019reinforcement}
Schmidhuber, J.
\newblock Reinforcement learning upside down: Don't predict rewards--just map
  them to actions.
\newblock \emph{arXiv preprint arXiv:1912.02875}, 2019.

\bibitem[Schrittwieser et~al.(2020)Schrittwieser, Antonoglou, Hubert, Simonyan,
  Sifre, Schmitt, Guez, Lockhart, Hassabis, Graepel,
  et~al.]{schrittwieser2020mastering}
Schrittwieser, J., Antonoglou, I., Hubert, T., Simonyan, K., Sifre, L.,
  Schmitt, S., Guez, A., Lockhart, E., Hassabis, D., Graepel, T., et~al.
\newblock Mastering atari, go, chess and shogi by planning with a learned
  model.
\newblock \emph{Nature}, 588\penalty0 (7839):\penalty0 604--609, 2020.

\bibitem[Schulman et~al.(2017)Schulman, Wolski, Dhariwal, Radford, and
  Klimov]{schulman17ppo}
Schulman, J., Wolski, F., Dhariwal, P., Radford, A., and Klimov, O.
\newblock Proximal policy optimization algorithms.
\newblock \emph{CoRR}, abs/1707.06347, 2017.
\newblock URL \url{http://arxiv.org/abs/1707.06347}.

\bibitem[Singh et~al.(2005)Singh, Barto, and Chentanez]{singh2005intrinsically}
Singh, S., Barto, A.~G., and Chentanez, N.
\newblock Intrinsically motivated reinforcement learning.
\newblock Technical report, MASSACHUSETTS UNIV AMHERST DEPT OF COMPUTER
  SCIENCE, 2005.

\bibitem[Stolle \& Precup(2002)Stolle and Precup]{stolle2002learning}
Stolle, M. and Precup, D.
\newblock Learning options in reinforcement learning.
\newblock In \emph{International Symposium on abstraction, reformulation, and
  approximation}, pp.\  212--223. Springer, 2002.

\bibitem[Todorov et~al.(2012)Todorov, Erez, and Tassa]{todorov2012mujoco}
Todorov, E., Erez, T., and Tassa, Y.
\newblock Mujoco: A physics engine for model-based control.
\newblock In \emph{2012 IEEE/RSJ International Conference on Intelligent Robots
  and Systems}, pp.\  5026--5033. IEEE, 2012.

\bibitem[Tucker et~al.(2018)Tucker, Bhupatiraju, Gu, Turner, Ghahramani, and
  Levine]{tucker2018mirage}
Tucker, G., Bhupatiraju, S., Gu, S., Turner, R., Ghahramani, Z., and Levine, S.
\newblock The mirage of action-dependent baselines in reinforcement learning.
\newblock In \emph{International conference on machine learning}, pp.\
  5015--5024. PMLR, 2018.

\bibitem[Uchendu et~al.(2022)Uchendu, Xiao, Lu, Zhu, Yan, Simon, Bennice, Fu,
  Ma, Jiao, et~al.]{uchendu2022jump}
Uchendu, I., Xiao, T., Lu, Y., Zhu, B., Yan, M., Simon, J., Bennice, M., Fu,
  C., Ma, C., Jiao, J., et~al.
\newblock Jump-start reinforcement learning.
\newblock \emph{arXiv preprint arXiv:2204.02372}, 2022.

\bibitem[Vaswani et~al.(2017)Vaswani, Shazeer, Parmar, Uszkoreit, Jones, Gomez,
  Kaiser, and Polosukhin]{vaswani2017attention}
Vaswani, A., Shazeer, N., Parmar, N., Uszkoreit, J., Jones, L., Gomez, A.~N.,
  Kaiser, {\L}., and Polosukhin, I.
\newblock Attention is all you need.
\newblock In \emph{Advances in neural information processing systems}, pp.\
  5998--6008, 2017.

\bibitem[Veeriah et~al.(2018)Veeriah, Oh, and Singh]{veeriah2018many}
Veeriah, V., Oh, J., and Singh, S.
\newblock Many-goals reinforcement learning.
\newblock \emph{arXiv preprint arXiv:1806.09605}, 2018.

\bibitem[Wu et~al.(2019)Wu, Tucker, and Nachum]{wu2019behavior}
Wu, Y., Tucker, G., and Nachum, O.
\newblock Behavior regularized offline reinforcement learning.
\newblock \emph{arXiv preprint arXiv:1911.11361}, 2019.

\bibitem[Yu et~al.(2021{\natexlab{a}})Yu, Velu, Vinitsky, Wang, Bayen, and
  Wu]{yu2021surprising}
Yu, C., Velu, A., Vinitsky, E., Wang, Y., Bayen, A., and Wu, Y.
\newblock The surprising effectiveness of mappo in cooperative, multi-agent
  games.
\newblock \emph{arXiv preprint arXiv:2103.01955}, 2021{\natexlab{a}}.

\bibitem[Yu et~al.(2021{\natexlab{b}})Yu, Kumar, Chebotar, Hausman, Levine, and
  Finn]{yu2021conservative}
Yu, T., Kumar, A., Chebotar, Y., Hausman, K., Levine, S., and Finn, C.
\newblock Conservative data sharing for multi-task offline reinforcement
  learning.
\newblock \emph{Advances in Neural Information Processing Systems}, 34,
  2021{\natexlab{b}}.

\bibitem[Zhao et~al.(2019)Zhao, Sun, and Tresp]{zhao2019maximum}
Zhao, R., Sun, X., and Tresp, V.
\newblock Maximum entropy-regularized multi-goal reinforcement learning.
\newblock In \emph{International Conference on Machine Learning}, pp.\
  7553--7562. PMLR, 2019.

\end{thebibliography}
